\documentclass[10pt]{article}

\usepackage[utf8]{inputenc}
\usepackage[T1]{fontenc}

\usepackage{epsf}
\usepackage{amsmath}

\allowdisplaybreaks

\usepackage[showframe=false]{geometry}
\usepackage{changepage}

\usepackage{epsfig}
\usepackage{amssymb}

\usepackage{amsthm}
\usepackage{setspace}
\usepackage{cite}
\usepackage{mcite}

\usepackage{algorithmic}  
\usepackage{algorithm}

\usepackage{shadow}
\usepackage{fancybox}
\usepackage{fancyhdr}

\usepackage{color}
\usepackage[usenames,dvipsnames,svgnames,table]{xcolor}
\newcommand{\bl}[1]{\textcolor{blue}{#1}}

\definecolor{mypurple}{rgb}{.4,.0,.5}

\usepackage[hyphens]{url}

\usepackage[colorlinks=true,
            linkcolor=black,
            urlcolor=blue,
            citecolor=purple]{hyperref}

\usepackage{breakurl}

\def\y{{\bf y}}

\def\x{{\bf x}}

\def\x{{\mathbf x}}

\def\x{{\bf x}}
\def\y{{\bf y}}
\def\z{{\bf z}}

\def\a{{\bf a}}

\def\h{{\bf h}}

\def\cL{{\mathcal L}}

\def\cG{{\mathcal G}}

\def\be{\begin{equation}}
\def\ee{\end{equation}}
\def\ba{\left[\begin{array}}
\def\ea{\end{array}\right]}

\def\x{{\bf x}}
\def\y{{\bf y}}
\def\z{{\bf z}}

\def\a{{\bf a}}

\def\1{{\bf 1}}

\def\g{{\bf g}}
\def\0{{\bf 0}}

\def\erfc{\mbox{erfc}}

\def\mR{{\mathbb R}}

\def\mE{{\mathbb E}}

\def\mP{{\mathbb P}}

\def\lp{\left (}
\def\rp{\right )}

\def\y{{\bf y}}

\def\x{{\bf x}}

\def\x{{\mathbf x}}

\def\x{{\bf x}}
\def\y{{\bf y}}
\def\z{{\bf z}}

\def\a{{\bf a}}

\def\h{{\bf h}}

\def\be{\begin{equation}}
\def\ee{\end{equation}}
\def\ba{\left[\begin{array}}
\def\ea{\end{array}\right]}

\def\x{{\bf x}}
\def\y{{\bf y}}
\def\z{{\bf z}}

\def\a{{\bf a}}

\def\R{{\bf R}}

\def\({\left (}
\def\){\right )}

\def\1{{\bf 1}}

\def\g{{\bf g}}
\def\0{{\bf 0}}

\usepackage{xcolor}
\usepackage{color}

\definecolor{darkgreen}{rgb}{0, 0.4,0}

\definecolor{purplebrown}{rgb}{0.5,0.1,0.6}

\definecolor{ultclupcol}{rgb}{0.1,0.5,0.5}

\definecolor{mytrycolor}{rgb}{0.5,0.7,0.2}

\definecolor{ultclupcola}{rgb}{.5,0,.5}

\definecolor{shadebrown}{rgb}{0.1,0.1,0.9}
\definecolor{lightblue}{rgb}{0.2,0,1}

\usepackage{fancybox}
\usepackage{graphicx}
\usepackage{epstopdf}
\usepackage{epsfig}
\usepackage{wrapfig}
\usepackage{subfigure}

\usepackage{xcolor}
\usepackage{tcolorbox}

\newtcbox{\xmybox}{on line,
arc=7pt,
before upper={\rule[-3pt]{0pt}{10pt}},boxrule=0pt,
boxsep=0pt,left=6pt,right=6pt,top=0pt,bottom=0pt,enhanced, coltext=blue, colback=white!10!yellow}

\newtcbox{\xmyboxa}{on line,
arc=7pt,
before upper={\rule[-3pt]{0pt}{10pt}},boxrule=0pt,
boxsep=0pt,left=6pt,right=6pt,top=0pt,bottom=0pt,enhanced, colback=white!10!yellow}

\newtcbox{\xmyboxb}{on line,
arc=7pt,
before upper={\rule[-3pt]{0pt}{10pt}},boxrule=1pt,colframe=darkgreen!100!blue,
boxsep=0pt,left=6pt,right=6pt,top=0pt,bottom=0pt,enhanced, colback=white!10!yellow}

\newtcbox{\xmyboxc}{on line,
arc=7pt,
before upper={\rule[-3pt]{0pt}{10pt}},boxrule=.7pt,colframe=blue!100!blue,
boxsep=0pt,left=6pt,right=6pt,top=0pt,bottom=0pt,enhanced, coltext=blue, colback=white!10!yellow}

\newtcbox{\xmytboxa}{on line,
arc=7pt,
before upper={\rule[-3pt]{0pt}{10pt}},boxrule=.0pt,colframe=pink!50!yellow,
boxsep=0pt,left=6pt,right=6pt,top=0pt,bottom=0pt,enhanced, coltext=white, colback=blue!40!red}

\newtcbox{\xmytboxb}{on line,
arc=7pt,
before upper={\rule[-3pt]{0pt}{10pt}},boxrule=.0pt,colframe=pink!50!yellow,
boxsep=0pt,left=6pt,right=6pt,top=0pt,bottom=0pt,enhanced, coltext=white, colback=white!40!green}

\makeatletter
\newcommand\subsubsubsection{\@startsection{paragraph}{4}{\z@}{-2.5ex\@plus -1ex \@minus -.25ex}{1.25ex \@plus .25ex}{\normalfont\normalsize\bfseries}}
\newcommand\subsubsubsubsection{\@startsection{subparagraph}{5}{\z@}{-2.5ex\@plus -1ex \@minus -.25ex}{1.25ex \@plus .25ex}{\normalfont\normalsize\bfseries}}
\makeatother

\newtheorem{theorem}{Theorem}

\newtheorem{lemma}{Lemma}

\begin{document}

\begin{singlespace}

\title{Phase transition of \emph{descending} phase retrieval algorithms   
}
\author{
\textsc{Mihailo Stojnic
\footnote{e-mail: {\tt flatoyer@gmail.com}} }}
\date{}
\maketitle

\centerline{{\bf Abstract}} \vspace*{0.1in}

We study theoretical limits of \emph{descending} phase retrieval algorithms. Utilizing \emph{Random duality theory}  (RDT) we develop a generic program that allows statistical characterization of various algorithmic performance metrics. Through these we identify the concepts of \emph{parametric manifold} and its \emph{funneling points} as key mathematical objects that govern the underlying algorithms' behavior. An isomorphism between single funneling point manifolds and global convergence of descending algorithms is established. The structure and shape of the parametric manifold as well as its dependence on the sample complexity are studied through both plain and lifted RDT. Emergence of a phase transition is observed. Namely, as sample complexity increases, parametric manifold transitions from a multi to a single funneling point structure. This in return corresponds to a transition from the scenarios where descending algorithms generically fail to the scenarios where they succeed in solving phase retrieval.  We also develop and implement a practical algorithmic variant that in a hybrid alternating fashion combines a barrier and a plain gradient descent. Even though the theoretical results are obtained for infinite dimensional scenarios (and consequently non-jittery parametric manifolds), we observe a  strong agrement between theoretical and simulated phase transitions predictions for fairly small dimensions on the order of a few hundreds..

\vspace*{0.25in} \noindent {\bf Index Terms: Phase retrieval; Descending algorithms; Random duality theory}.

\end{singlespace}

\section{Introduction}
\label{sec:back}

For a given unit norm $\bar{\x}\in\R^n$ we consider the following \emph{phaseless} collection of linear measurements
\begin{eqnarray}
   \y &=& |A\bar{\x}|^2, \label{eq:inteq1}
\end{eqnarray}
where $A\in\mR^{m\times n}$. Recovering $\bar{\x}$ provided that one has access to $\y$ and $A$ assumes solving the following inverse problem
\begin{eqnarray}
\mbox{find} & & \x \nonumber \\
   \mbox{subject to} & & |A\x|^2=\y (= |A\bar{\x}|^2). \label{eq:inteq2}
\end{eqnarray}
The above feasibility problem (basically a system of quadratic equations) is the real version of the well known \emph{phase retrieval} (PR) problem. PR paradigm is the key mathematical feature  of a host of practical data acquisitions and recovery protocols in signal and image processing, particularly in x-ray crystallography, ptychography, holography, modern optics, and so on. It typically appears in scenarios where the nature of the measurement gathering process is such that acquiring signal's phase is infeasible.

Checking all possible options for signs of $\y$ and solving the residual linear problem is a simple procedure that can be used to solve (\ref{eq:inteq2}). Under solving (\ref{eq:inteq2}), we in this paper always assume solving it up to \emph{global} phase ambiguity as $|A\bar{\x}|^2=|A(-\bar{\x})|^2$ implies that both $\bar{\x}$ and $-\bar{\x}$ are admissible solutions. Moreover, we generically consider practically relevant so-called \emph{solvable} scenarios, i.e. scenarios where $A$ is such that besides $\pm \bar{\x}$  there are no other solutions of (\ref{eq:inteq2}). Due to its combinatorial nature, naively checking all sign combinations of $\y$ becomes computationally prohibitive as $n$ and $m$ increase.  Our interest in this paper is precisely in the large dimensional \emph{linear/proportional regime} which assumes
\begin{eqnarray}
 \alpha \triangleq \lim_{n\rightarrow\infty} \frac{m}{n}, \label{eq:inteq3}
\end{eqnarray}
with $\alpha$ -- the so-called \emph{oversampling or sample complexity ratio} -- remaining constant as the dimensions grow. Clearly, exhaustively searching over signs of $\y$ incurs an \emph{exponential} complexity of the order of $O(2^m)$ and as such is unacceptable from the practical point of view (in imaging PR applications $n$ can easily be on the order of tens of thousands). The need for computationally efficient recovery schemes has been recognized from the early PR days back in the first half of the last century. Consequently, a rather vast and fairly systematic algorithmic theory has been developed over last several decades. The mathematical research however particularly intensified over the last 10-15 years in large part due to the appearance of several algorithmic and theoretical breakthroughs  \cite{CandesSV13,CandesESV13,CandesLS15}. We here focus on one of these breakthroughs, the so-called Wirtinger flow or gradient algorithm and analyze its performance on a so-called \emph{quantitative} (high precision \emph{phase transition}) level (to be a bit more precise, we study a more general class of \emph{all descending} algorithms which encompasses Wirtinger flow/gradient descent based ones as special cases).

\subsection{Related prior work}
\label{sec:relwork}

\noindent $\star$ \underline{\emph{Phase retrieval relevance}:} As mentioned above, mathematical problems that we study are directly connected to the practically relevant phase retrieval (PR) problems. These problems date back to the early days of x-ray crystallography \cite{Harrison93,Millane90,Millane06}. As  x-ray crystallography extended to noncrystalline materials, PR became a key integral analytical/matheamtical  part of coherent diffraction, ptychography, optical, astronomical or microscopic imaging (see, e.g., \cite{Thibault08,Hurt89,KST1995,Miao1999ExtendingTM,ShechtmanECCMS15,MISE08,Bunk07,Walther01011963,Fienup87,Fienup78,Rod08,Dierolf10,BS79,Misell73}).  Further connections to digital holography \cite{Duadi11,Gabor48,Gabor65} , quantum physics \cite{Corb06,Hein13,HaahHJWY17}, blind deconvolution/demixing \cite{MWCC19,LLSW19,Jung17,ARJ13}, and many other fields emerged over the previous several decades as well. Our main interest is in mathematical aspects of PR problems and on that front two groups of PR considerations established themselves so far as of predominant interest: \textbf{\emph{(i)}} The first one that typically relates to ``mathematical soundness'' of the problem and encompasses studying algebraic  characterizations of injectivity/uniqueness and stability  properties that stem from the inherent PR phase ambiguity (see, e.g., \cite{Conca15,Balan06,Balan09,Bande14,Vinz15} and references therein); and \textbf{\emph{(ii)}} The second one that relates to actual solving of the underlying optimization PR problem which in turn enables the recovery of the desired ``phaselessly'' observed object. As this paper is more aligned with the second group we below look in more detail through the relevant optimization techniques available in the literature and discuss some of the key milestones.

\noindent $\star$ \underline{\emph{Convexity based methods}:}  Since PR is directly related to compressed sensing (CS) one would naturally expect that convex relaxation optimization techniques utilized for solving various CS problems would be of use in handling PR analogues. This is indeed the case. As the key aspect of the phase retrieval non-convexity stems from the existence of a matrix rank-1 constraint, standard optimization  techniques typically suggest relaxing it to full rank. This is precisely what was done in  \cite{CandesSV13,CandesESV13} where a so-called Phaselift relaxation algorithm is introduced (for a closely related Phasecut relaxation -- basically a direct PR analogue to the famous maxcut -- see \cite{WaldspurgerdM15}). Moreover, \cite{CandesSV13,CandesESV13} showed that Phaselift is successful in (both exactly and stably) recovering original signal from phaseless observations provided that sample complexity satisfies $m=O(n\log(n))$ (see also, \cite{PHand17} for further Phaselift robustness discussion). \cite{CandesL14} then further improved sample complexity to $m=O(n)$ which is optimal up to a constant. Corresponding versions of these results when one instead of Gaussians relies on Fourier measurements were obtained in \cite{GFK17,CandesLS15b}. While the rank dropping semi-definite programming (SDP) relaxations tend to be among the tightest convex ones, they rely on solving the residual SDP which (due to dimensions being lifted from $n$ to $n^2$) can pose a serious numerical challenge. As a numerically more amenable alternative, \cite{GoldsteinS18,BahmaniR16} introduced the so-called PhaseMax and showed that it can recover the original signal provided that the starting point of the algorithm (initializer) is sufficiently close to it (see \cite{HandV16,ChenCandes17} for slight improvements which ultimately design the initializer so that the sample complexity is $m=O(n)$ and \cite{SalehiAH18} for precise PhaseMax analysis). Another convex multi-iterative PhaseMax alternative, PhaseLamp, was introduced in \cite{DhifallahTL17} and shown to allow for better both theoretical and algorithmic properties than PhaseMax. Further structuring of the unknown signal (say sparsity, low-rankness, positivity and so on) can easily be incorporated in convex methods. One basically proceeds along the lines of what is usually done in standard compressed sensing (see, e.g. \cite{Ohlsen12,LiVor13,KeungRTi17}). For example, \cite{LiVor13} shows that sample complexity $m=O(k^2\log(n))$ suffice for successful PR recovery of a $k$-sparse vector which is even in $m\ll n$ (a nonlinear regime that is not of our prevalent interest) surprisingly weaker than the corresponding  $m=O(k\log(n))$  compressed sensing result. Various adaptations/upgrades of standard convex techniques are possible as well when additional structuring is taken into account (for two-stage approaches and randomized Kaczmarz adaptations see \cite{JaganathanOH17,IwenVW17} and \cite{TV19,KWei15}, respectively).

\noindent $\star$ \underline{\emph{Non-convex methods}:} Somewhat surprisingly, within the PR context utilization  of non-convex  methods significantly predates the convex ones. The most typical early PR algorithmic considerations relate to  the so-called alternating minimization \cite{Gerch72,Fienup82,Fienup87}. One of the big breakthroughs arrived, however, when \cite{CandesLS15} uncovered that the simple gradient (the authors called it Wirtinger flow due to direct dealing with complex derivatives instead of their real equivalents) often significantly outperforms all known techniques (for a closely related variant called amplitude flow, see, e.g., \cite{WangGE18} and  for further modifications, e.g.,  \cite{CLM16}). Moreover, not only did the basic gradient turn out to be numerically substantially superior compared to the SDP relaxations, but the accompanying theoretical analysis of  \cite{CandesLS15} also revealed that it behaves in such a way while allowing for $m=O(n\log(n))$ sample complexity (the sparse variants of Wirtinger and
amplitude flow \cite{YuanWW19,WangGE18}, follow the suit of the convex methods and allow for $m=O(k^2\log(n))$; on the other hand, \cite{Soltanolkotabi19} in a way breaks the $k^2$ barrier and shows that $m=O(k\log(n))$ suffices for local Wirtinger flow convergence). This is particularly surprising given that the preliminary considerations \cite{SunQW18,HandLV18} of the underlying objective revealed a very strong deviations from convexity. Another non-convex
classical variant, Alternating direction method of multipliers (ADMM) (basically Douglas-Rachford within the PR context), was introduced in \cite{FannjiangZ20} and shown to perform very well (see, also \cite{Netrapalli0S15,Waldspurger18} for theoretical characterizations of classical alternating minimization projections). Given the promise the use of generative models and deep learning in compressed sensing have shown
\cite{LeiJDD19,JordanD20,BoraJPD17,DaskalakisRZ20a}, in \cite{HandLV18} Deep (sparse) phase retrieval (DPR) concept was put forth. Empirical results fairly similar to those obtained in \cite{BoraJPD17} for deep learning compressed sensing were obtained in \cite{HandLV18}, basically indicating that the deep nets have a strong potential to substantially improve upon any of the known non-machine learning based techniques. Moreover, \cite{HandLV18} proved that a constant expansion $d$-layer net can achieve such a performance while allowing for sample complexity $m=O(kd\log(n))$, which for a constant $d$ allows to break the $k^2\log(n)$ barrier mentioned above and achieve the  $k\log(n)$ as predicated by drawing the analogy with standard compressed sensing. However, deep learning also has some limitations as well. Nonzero errors are hard to achieve (which effectively limits the precision in the so-called realizable or noiseless scenarios), allowed problem dimensions are often restrictive, and a frequent (time-consuming) retraining might be needed depending on the application. Precise theoretical characterizations are much harder to obtain as well and despite a very strong progress recently achieved in that direction (see, \cite{MBBDN23,Stojnicinjdeeprelu24,Stojnicinjrelu24}), utilization of deep nets within the phase retrieval context is still a bit distant from indeed being generically superior as one naturally expects. On the other hand, generically more competitive with the above nonconvex methods are heuristic adaptations of the classical approximate message passing (AMP) algorithm (initially designed for large scale compressed sensing applications in  \cite{DonMalMon09}). While such methods and their analyses may heavily depend on model assumptions, an excellent empirical performance is demonstrated in \cite{SchniterR15}.

Clearly, all other parameters equal, the lowest sample complexity is the most desirable feature of any PR algorithm.  As the above suggests, the best sample complexity among the convex methods is typically achieved by the SDP relaxations, i.e. by Phaselift and its derivatives. On the other hand the fastest are those that relate to linear programming and PhaseMax. A natural question is whether one can design algorithms that achieve both good sample complexity and quick convergence without consuming way too much memory. The above discussed nonconvex methods to a large degree do achieve that. However, to be able to provide more convincing arguments that this is indeed true, one can not be satisfied with standard \emph{qualitative} performance assessments and instead needs to resort to more precise \emph{quantitative} ones.  This is precisely what we do in this paper. Before we proceed with presentation of our results,  we below discuss results obtained in these directions throughout the literature so far.

\noindent $\star$  \underline{\emph{``Qualitative'' vs ``quantitative'' performance characterizations:}}  Most of the results discussed above relate to qualitative performance characterizations. As such they provide a good intuitive hint as to how  algorithms behave either on their own or in a comparison with others. However, they are far away from being precise. Obtaining precise, \emph{quantitative}, characterizations is usually much harder and the known results are expectedly  scarcer. In \cite{TakahashiK22} a replica analysis of vector AMP (VAMP) algorithms \cite{SchniterRF16,RanganSF17} is presented. As in \cite{TakahashiK22}, a Bayesian inference context is considered in \cite{MaillardLKZ20} and a large set of replica predictions shown to be correct (for a real version of the complex counterpart from \cite{MaillardLKZ20} see \cite{BarbierKMMZ18}; Bayesian context is also utilized (albeit differently) in  \cite{StrSag25}, where an excellent two stage phase-selection algorithmic procedure is proposed and analyzed via replica methods). Since Bayesian context in general assumes prior knowledge of the signal's statistics, statistical identicalness over its components, and a perfect knowledge of the so-called channel (or posterior) it can not be fairly compared to other methods mentioned above. On the plus side though, it does contain precise, phase transition types of, results that are unreachable by any of the methods discussed above.

\noindent $\star$  \underline{\emph{Initializers:}}  Given the role initializers play in the non-convex methods, we also point out a few excellent results obtained in those directions. Namely, in \cite{LuL17} (see also \cite{LuoAL19,LuLi20} for further extensions) a precise ``residual overlap vs oversampling ratio'' characterization is provided for the so-called nonnegative diagonal preprocessing. It further revealed existence of a sharp oversampling ratio transition between phases where the overlap (the unit norm normalized inner-product between the estimate and the true signal) is zero or nonzero. This immediately implied direct consequences for phase retrieval algorithms as their ability to reach the global optimum  is typically predicated on the use of nonzero overlap producing initializers. Results of \cite{LuL17} were then extended in \cite{MondelliM19} to negative diagonals (and complex domain) and the so-called optimal initializers were obtained as well. Relying on such initializers, \cite{MondelliM19} then went further and proved the critical sample complexity above which a nonzero overlap appears. In \cite{LuoAL19} the optimality of spectral  preprocessing discussed in \cite{MondelliM19} was proven for any sample complexity. While  \cite{LuL17,MondelliM19} relied on Gaussian measurements, \cite{MaDXMW21} considered orthogonal ones and utilized an Expectation Propagation paradigm to precisely characterize ``residual overlap vs oversampling ratio'' (even more general classes of measurement matrices and models were considered via statistical physics tools in \cite{MaillardKLZ21,AubinLBKZ20}). Results of \cite{MaDXMW21} were also established as mathematically fully rigorous in \cite{DudejaB0M20}.

In what follows, we consider non-convex methods and their algorithmic phase transitions. Before getting into technical details, we first briefly summarize our key results.

\subsection{Our contributions}
\label{sec:contrib}

As stated above, we focus on theoretical  performance characterizations of non-convex optimization methods utilized in algorithmic phase retrieval. Of our main interest are the so-called \emph{descending} algorithms (among others, they encompass all forms of the famous gradient descent). We emphasize again that we consider a statistical \emph{linear/proportional} high-dimensional regime where the length of the unknown vector is $n$ and the number of phaseless measurements (sample complexity) is $m$ with $\alpha=\lim_{n\rightarrow \infty} \frac{m}{n}$ remaining constant as $n$ grows. Also we present the results for the real valued scenario (however, all presented concepts directly extend to the complex case; as the technical details are different and require problem specific adjustments we defer presenting them to a separate companion paper).

\begin{itemize}
\item We design a convenient \emph{fundamental  phase retrieval optimization} (f-pro) problem formulation and recognize the importance of studying behavior of its objective (see part 1) in Section \ref{sec:ubrdt}).
  \item A \emph{Random duality theory}  (RDT) based generic program is then developed for statistical studying of f-pro (see parts 2) and 3) in Section \ref{sec:ubrdt}).
  \item  The RDT program allows to identify the concept of  \emph{parametric manifold} as a key mathematical object impacting descending algorithms abilities to solve f-pro. In particular, two parameters  --  \textbf{\emph{1)}} overlap between the algorithmic solution and the true signal and \textbf{\emph{2)}} the $\ell_2$ norm of the algorithmic solution -- are recognized as the manifold's main ingredients. In particular, their interconnections determine the manifold's shape (see Section \ref{sec:algimp}).
  \item  The concept of manifold's so-called \emph{funneling points} is put forth as well and an isomorphism between
 single funnelling point manifolds and successful solving of f-pro via  \emph{descending} algorithms is established (see Section \ref{sec:algimp}).
   \item The way oversampling ratio impacts manifold's shape is numerically evaluated utilizing the developed RDT program (for visual illustrations see Figures \ref{fig:fig2}-\ref{fig:fig4}). Emergence of a phase transition is observed. In particular, as sample complexity increases, parametric manifold transitions from a multi to a single funneling point structure implying a  descending algorithms transition from generically failing  to generically succeeding in reaching  f-pro's global optimum.
   \item Due to a lack of the so-called strong deterministic duality, a lack of strong random duality is anticipated as well which implies a strictly bounding character of the plain RDT results from Section \ref{sec:ubrdt}. To remedy such an inevitability and further lift plain RDT estimates,  a powerful \emph{Lifted} RDT based program is developed (see Section \ref{sec:liftrdt} for theoretical consideration and Figures \ref{fig:fig5} and \ref{fig:fig6}  for visual illustrations of the ``\emph{lifting effect}'').
\item We implement a hybrid combination of a barrier gradient descent algorithm and a plain gradient descent and demonstrate that, despite relatively small dimensions ($n=300$) and strong anticipated jitteriness effects, its simulated phase transition is fairly close to the theoretical predictions (see Figure \ref{fig:fig7}).
\item The critical role the norm of the unknown vector (one of the two manifold's parameters) plays in both hybrid and plain gradient optimization processes is revealed and discussed in detail as well (see Section \ref{sec:plaingrad}).
    \item Elegant closed form analytical considerations made theoretical analyses of Sections \ref{sec:2lay} and \ref{sec:liftrdt} more convenient to conduct with plain magnitudes in the objective. On the other hand, smoothness of the \emph{squared} magnitudes derivatives made practical running easier. To ensure that practical implementations have their corresponding theoretical counterparts, we in Sections \ref{sec:sqadj} and \ref{sec:sqliftrdt} conducted RDT and lifted RDT analyses for squared magnitudes as well. While the theoretical considerations are way less elegant and the residual evaluations are much more involved (and consequently prone to potential numerical inaccuracies) all main theoretical conclusions from Sections \ref{sec:2lay} and \ref{sec:liftrdt} translate to the corresponding squared magnitudes related ones (see Figures \ref{fig:fig8}-\ref{fig:fig12}).
\end{itemize}

\section{\emph{Real} phase retrieval}
 \label{sec:2lay}

Relying on the observations from the previous sections, we proceed in a fashion similar to \cite{Stojnicinjrelu24,StojnicGardGen13,StojnicICASSP10var,StojnicCSetam09} and connect studying phase retrieval algorithmic and theoretical properties to the analysis of random optimization problems (rops). In particular, we observe that the following rop is of key interest
 \begin{eqnarray}
 {\mathcal R}(A): \qquad \qquad    \min_{\x,\z} & & \||A\bar{\x}|-|\z| \|_2^2\nonumber \\
  \mbox{subject to} & &  A\x=\z. \label{eq:ex1a4}
\end{eqnarray}
We find it convenient to further set
 \begin{eqnarray}
\hspace{-.4in}\bl{\textbf{\emph{Fundamental phase-retrieval optimization (f-pro):}}} \qquad\qquad  \xi(c,x) \triangleq \min_{\x,\z} & & \||A\bar{\x}|-|\z| \|_2^2\nonumber \\
  \mbox{subject to} & &  A\x=\z \nonumber \\
  & & \x^T\bar{\x}=x \nonumber \\
  & & \|\x\|_2^2=c. \label{eq:ex1a4a0}
\end{eqnarray}
One should note that (\ref{eq:ex1a4}) and (\ref{eq:ex1a4a0}) slightly differ from the classical phase retrieval formulations typically seen in the literature. Namely, the common practice is to consider \emph{squared magnitudes}, $|A\x|^2$ and $|\z|^2$, in the objective. Throughout the literature such scenarios effectively assume that one works with \emph{intensity} measurements. On the other hand, taking just (non-squared) magnitudes typically corresponds to work with \emph{amplitude} measurements. There is really not that much conceptual difference between the two options. Throughout the paper we consider the \emph{non-squared} option as it allows for more elegant and numerically efficient treatment of our analytical considerations. In Sections \ref{sec:sqadj} and \ref{sec:sqliftrdt} the \emph{non-squared} results are complemented with the corresponding \emph{squared} ones.

Let $\hat{\x}$ be a solution of (\ref{eq:ex1a4a0}). One first trivially observes that for $c=x=1$,  $\xi(1,1)=0$ and $\hat{\x}=\bar{\x}$, i.e, the solution of the above optimization matches the desired $\bar{\x}$ (to avoid trivial phase ambiguity, throughout the presentation we consider only $x>0$; any of the key results that we obtain for $x>0$ holds symmetrically for $x<0$ as well). Two problems are then of interest: \textbf{\emph{(i)}} Is $\xi(1,x)>0$ for $x\neq 1$? and \textbf{\emph{(ii)}} Can a computationally efficient algorithm that solves (\ref{eq:ex1a4a0}) be designed? It is not that difficult to see that a positive answer to the first question ensures the so-called uniqueness and, consequently, \emph{theoretical} solvability of (\ref{eq:ex1a4a0}). On the other hand, a positive answer to the second question ensures a \emph{practical} solvability of (\ref{eq:ex1a4a0}). As the behavior of $\xi(c,x)$ critically impacts answering both of these questions, we study it in the following section in more detail. In particular, we consider a standard normal $A$ and develop a generic \emph{Random duality theory} (RDT) based analytical program that allows to characterize f-pro.

\subsection{Analysis of $\xi(c,x)$ via Random Duality Theory (RDT)}
\label{sec:ubrdt}

Before proceeding with a detailed analysis we briefly summarize the main RDT principles developed in a long series of work \cite{StojnicCSetam09,StojnicICASSP10var,StojnicISIT2010binary,StojnicICASSP10block,StojnicRegRndDlt10,StojnicGenLasso10}. We then continue by discussing how each of them can be implemented within the context of our interest here.

\vspace{-.0in}\begin{center}
 \begin{tcolorbox}[title={\small Summary of the RDT's main principles} \cite{StojnicCSetam09,StojnicRegRndDlt10}]
\vspace{-.15in}
{\small \begin{eqnarray*}
 \begin{array}{ll}
\hspace{-.19in} \mbox{1) \emph{Finding underlying optimization algebraic representation}}
 & \hspace{-.0in} \mbox{2) \emph{Determining the random dual}} \\
\hspace{-.19in} \mbox{3) \emph{Handling the random dual}} &
 \hspace{-.0in} \mbox{4) \emph{Double-checking strong random duality.}}
 \end{array}
  \end{eqnarray*}}
\vspace{-.2in}
 \end{tcolorbox}
\end{center}\vspace{-.0in}

\noindent To ensure neatness of the presentation, all key results (simple as well as more complicated ones) are formulated as lemmas or theorems.

\vspace{.1in}

\noindent \underline{1) \textbf{\emph{Algebraic phase retrieval characterization:}}}  We start by observing that  rotational invariance of $A$ allows to
 rotate (without loss of generality) $\bar{\x}$ so that it becomes $\bar{\x}=[\|\bar{\x}\|_2,0,\dots,0 ]^T$. One can then
rewrite (\ref{eq:ex1a4a0}) as
 \begin{eqnarray}
  \xi(c,x) = \min_{\x,\z} & & \||A_{:,1}|\|\bar{\x}\|_2-|\z| \|_2^2\nonumber \\
  \mbox{subject to} & &  A\x=\z \nonumber \\
  & & \x^T\bar{\x}=\x_1\|\bar{\x}\|_2= x \nonumber \\
  & & \|\x\|_2^2=c, \label{eq:rdteq0a0}
\end{eqnarray}
where $A_{:,i}$ stands for the $i$-th column of $A$. We al;so find it convenient to set $r\triangleq \sqrt{c-x^2}$. It is then easy to see that restriction to $\|\bar{\x}\|_2=1$ can also be done without loss of generality. One can then rewrite (\ref{eq:rdteq0a0}) as
 \begin{eqnarray}
  \xi(c,x) = \min_{\x,\z} & & \||A_{:,1}|-|\z| \|_2^2\nonumber \\
  \mbox{subject to} & &  A\x=A_{:,1}x + A_{:,2:n}\x_{2:n} = \z \nonumber \\
   & & \sum_{i=2}^{n}\x_i=c-x^2=r^2, \label{eq:rdteq0a1}
\end{eqnarray}
Writing Lagrangian also gives
 \begin{eqnarray}
  \xi(c,x) = \min_{\x,\z} \max_{\y}   & & \||A_{:,1}|-|\z| \|_2^2
  +\y^TA_{:,1}x + \y^TA_{:,2:n}\x_{2:n} -\y^T \z \nonumber \\
  \mbox{subject to}
     & & \sum_{i=2}^{n}\x_i=c-x^2=r^2. \label{eq:rdteq0a2}
\end{eqnarray}
Setting $\g^{(0)}=A_{:,1}$, one obtains  a more compact form of (\ref{eq:rdteq0a2})
 \begin{eqnarray}
  \xi(c,x) = \min_{\|\x_{2:n}\|_2=r,\z} \max_{\y}  \lp \||\g^{(0)}|-|\z| \|_2^2
  +\y^T \g^{(0)}x + \y^TA_{:,2:n}\x_{2:n} -\y^T \z \rp. \label{eq:rdteq0a3}
\end{eqnarray}
The above  is a rather useful characterization of the fundamental phase retrieval optimization (f-pro). Together with its implications regarding theoretical solvability of the phase retrieval, it is summarized in the following lemma.

\begin{lemma} Consider real phase retrieval (PR) problem with $n$ unknowns and sample complexity $m$. Let $A\in\mR^{m\times n}$, $\g^{(0)}\triangleq A_{:,1}$, and assume a high-dimensional linear (proportional) regime,
with the over-sampling ratio $\alpha=\lim_{n\rightarrow\infty}\frac{m}{n}$. Then, the PR is theoretically solvable (i.e., it has a unique (up to a global phase) solution) provided that
 \begin{equation}\label{eq:ta10}
 \forall x\neq 1 \qquad f_{rp}(1,x;A)>0,
\end{equation}
where
\begin{equation}\label{eq:ta11}
f_{rp}(c,x;A)\triangleq \frac{1}{n}  \min_{\|\x_{2:n}\|_2=r,\z} \max_{\y}  \lp \||\g^{(0)}|-|\z| \|_2^2
  +\y^T \g^{(0)}x + \y^TA_{:,2:n}\x_{2:n} -\y^T \z \rp.
\end{equation}
    \label{lemma:lemma1}
\end{lemma}
\begin{proof}
Follows immediately from the above discussion and after recognizing that (\ref{eq:ta11}) is (\ref{eq:rdteq0a3})  cosmetically  scaled by $1/n$.
\end{proof}

The optimization problem on the right hand side of (\ref{eq:ta11}) is the so-called \emph{random primal}. The corresponding \emph{random dual} is determined next.

\vspace{.1in}
\noindent \underline{2) \textbf{\emph{Determining the random dual:}}} As is typical within the RDT, the measure concentration is utilized as well. This practically means that for any fixed $\epsilon >0$,  one has (see, e.g. \cite{StojnicCSetam09,StojnicRegRndDlt10,StojnicICASSP10var})
\begin{equation*}
\lim_{n\rightarrow\infty}\mP_{A}\left (\frac{|f_{rp}(c,x;A)-\mE_{A}(f_{rp}(c,x;A)|}{\mE_{A}(f_{rp}(c,x;A)}>\epsilon\right )\longrightarrow 0.\label{eq:ta15}
\end{equation*}
Another key ingredient of the RDT machinery is the following, so-called random dual, theorem.
\begin{theorem} Assume the setup of Lemma \ref{lemma:lemma1}. Let the elements of $A\in\mR^{m\times n}$ ($\g^{(0)}\in\mR^{m\times 1}$ and $A_{:,2:n}\in\mR^{m\times (n-1)}$), $\g^{(1)}\in\mR^{m\times 1}$, and  $\h^{(1)}\in\mR^{(n-1)\times 1}$  be iid standard normals. Consider two positive scalars $c$ and $x$  ($0\leq x \leq c$) and set $r\triangleq \sqrt{c-x^2}$. Let
\vspace{-.0in}
\begin{eqnarray}
\cG & \triangleq & \lp A,\g^{(1)},\h^{(1)}\rp = \lp\g^{(0)},A_{:,2:n},\g^{(1)},\h^{(1)}\rp  \nonumber \\
\phi(\x,\z,\y) & \triangleq &
 \lp \||\g^{(0)}|-|\z| \|_2^2
  +\y^T \g^{(0)}x   +  \y^T \g^{(1)}\|\x_{2:n}\|_2 + \lp \x_{2:n}  \rp^T\h^{(1)}\|\y\|_2  -\y^T  \z  \rp
\nonumber \\
 f_{rd}(c,x;\cG) & \triangleq &
\frac{1}{n}  \min_{\|\x_{2:n}\|_2=r,\z} \max_{\|\y\|_2=r_y,r_y>0}  \phi(\x,\z,\y)
  \nonumber \\
 \phi_0 & \triangleq & \lim_{n\rightarrow\infty} \mE_{\cG} f_{rd}(c,x;\cG).\label{eq:ta16}
\vspace{-.0in}\end{eqnarray}
One then has \vspace{-.02in}
\begin{eqnarray}
  \lim_{n\rightarrow\infty}\mP_{ A } (f_{rp} (c,x; A )   >  \phi_0)\longrightarrow 1,\label{eq:ta17a0}
\end{eqnarray}
and
\begin{eqnarray}
\hspace{-.3in}(\phi_0  > 0)   &  \Longrightarrow  & \lp \lim_{n\rightarrow\infty}\mP_{\cG}\lp \frac{\xi(c,x)}{n} = f_{rd}(c,x;\cG) >0 \rp \longrightarrow 1\rp
\quad  \Longrightarrow \quad \lp \lim_{n\rightarrow\infty}\mP_{ A } (f_{rp} (c,x; A )   >0)\longrightarrow 1 \rp  \nonumber \\
& \Longrightarrow & \lp \lim_{n\rightarrow\infty}\mP_{A} \lp \mbox{PR is (uniquely) solvable} \rp \longrightarrow 1\rp.\label{eq:ta17}
\end{eqnarray}
 \label{thm:thm1}
\end{theorem}\vspace{-.17in}
\begin{proof}
Follows immediately after conditioning on $\g^{(0)}$ and applying the Gordon's comparison theorem (see, e.g., Theorem B in \cite{Gordon88}). Gordon's theorem is a special case of a series of results Stojnic obtained in \cite{Stojnicgscomp16,Stojnicgscompyx16} (see Theorem 1, Corollary 1, and Section 2.7.2 in \cite{Stojnicgscomp16} as well as Theorem 1, Corollary 1, and Section 2.3.2 in \cite{Stojnicgscompyx16}).
\end{proof}

\vspace{.1in}
\noindent \underline{3) \textbf{\emph{Handling the random dual:}}} We follow the methodologies invented in \cite{StojnicCSetam09,StojnicICASSP10var,StojnicISIT2010binary,StojnicICASSP10block,StojnicRegRndDlt10}. To that end, we first solve the optimizations over $\x$ and $\y$ and obtain from (\ref{eq:ta16})
\begin{equation}
 f_{rd}(c,x;\cG) =
\frac{1}{n}  \min_{\z} \max_{r_y>0}
 \lp \||\g^{(0)}|-|\z| \|_2^2
  +\|\g^{(0)}x   +  \g^{(1)}r -\z\|_2 r_y - \|\h^{(1)} \|_2 r r_y \rp.
\label{eq:hrd1}
 \end{equation}
We then also have
\begin{eqnarray}
 f_{rd}(c,x;\cG)
  &  =  &
\frac{1}{n}  \min_{\z} \max_{r_y>0}
 \lp \||\g^{(0)}|-|\z| \|_2^2
  +\|\g^{(0)}x   +  \g^{(1)}r -\z\|_2^2 r_y - \|\h^{(1)} \|_2^2 r^2 r_y \rp \nonumber \\
  &  \geq   &
\frac{1}{n}  \max_{r_y>0} \min_{\z}
 \lp \||\g^{(0)}|-|\z| \|_2^2
  +\|\g^{(0)}x   +  \g^{(1)}r -\z\|_2^2 r_y - \|\h^{(1)} \|_2^2 r^2 r_y \rp
  \nonumber \\
  &  =   &
\frac{1}{n}  \max_{r_y>0} \min_{\z_i}
 \lp
\sum_{i=1}^{m}
\lp \||\g_i^{(0)}|-|\z_i| \|_2^2
  +\|\g_i^{(0)}x   +  \g_i^{(1)}r -\z_i\|_2^2 r_y \rp
  - \|\h^{(1)} \|_2^2 r^2 r_y \rp
    \nonumber \\
  &  =   &
\frac{1}{n}  \max_{r_y>0} \min_{\z_i}
 \lp
\sum_{i=1}^{m}
\lp \||\g_i^{(0)}|-|\z_i| \|_2^2
  +\||\g_i^{(0)}x   +  \g_i^{(1)}r| -|\z_i|\|_2^2 r_y \rp
  - \|\h^{(1)} \|_2^2 r^2 r_y \rp.
\label{eq:hrd2}
 \end{eqnarray}
Statistical identicalness over $i$ and concentrations then give
\begin{eqnarray}
 \phi_0 & \triangleq & \lim_{n\rightarrow\infty} \mE_{\cG} f_{rd}(\cG)
\geq
  \max_{r_y>0} \mE_{\cG}  \min_{\z_i} \cL_1(r_y),
\label{eq:hrd6}
 \end{eqnarray}
where
\begin{eqnarray}
\cL_1(r_y)
  &  = &
\alpha  \lp \||\g_i^{(0)}|-|\z_i| \|_2^2
  +\||\g_i^{(0)}x   +  \g_i^{(1)}r| -|\z_i|\|_2^2 r_y \rp
  -  r^2 r_y.
\label{eq:hrd7}
 \end{eqnarray}
To optimize over $\z_i$ one first finds the following derivative
\begin{eqnarray}
\frac{d\cL_1(r_y)}{d|\z_i|}
  &  = &
\alpha \lp -2(|\g_i^{(0)}|-|\z_i| )
  -2(|\g_i^{(0)}x   +  \g_i^{(1)}r| -|\z_i|) r_y\rp.
\label{eq:hrd7a0}
 \end{eqnarray}
 Equalling the above derivative to zero allows to determine the optimal $|\z_i|$
\begin{eqnarray}
|\hat{\z}_i|= \frac{1}{1+r_y} \lp |\g_i^{(0)}| + |\g_i^{(0)}x   +  \g_i^{(1)}r| r_y \rp.
\label{eq:hrd7a1}
 \end{eqnarray}
Combining  (\ref{eq:hrd7}) and  (\ref{eq:hrd7a1}) one further obtains
\begin{eqnarray}
\min_{\z_i} \cL_1(r_y)
  &  = &
\alpha  \lp \||\g_i^{(0)}|-|\hat{\z}_i| \|_2^2
  +\||\g_i^{(0)}x   +  \g_i^{(1)}r| -|\hat{\z}_i|\|_2^2 r_y \rp
  -  r^2 r_y \nonumber \\
  &  = &
\alpha
\lp
\frac{r_y^2}{(1+r_y)^2} \lp |\g_i^{(0)}| - |\g_i^{(0)}x   +  \g_i^{(1)}r|  \rp^2
  +\frac{1}{(1+r_y)^2} \lp |\g_i^{(0)}| - |\g_i^{(0)}x   +  \g_i^{(1)}r|  \rp^2
 r_y \rp
  -  r^2 r_y
  \nonumber \\
    &  = &
\alpha
\frac{r_y}{1+r_y} \lp |\g_i^{(0)}| - |\g_i^{(0)}x   +  \g_i^{(1)}r|  \rp^2
  -  r^2 r_y.
\label{eq:hrd7a2}
 \end{eqnarray}
 After setting
\begin{eqnarray}
f_q & \triangleq & \mE_{\cG}\lp |\g_i^{(0)}| - |\g_i^{(0)}x   +  \g_i^{(1)}r|  \rp^2,
\label{eq:hrd7a2a0}
 \end{eqnarray}
 one easily writes
\begin{eqnarray}
 \mE_{\cG} \min_{\z_i} \cL_1(r_y)
     &  = &
\alpha
\frac{r_y}{1+r_y}  \mE_{\cG}\lp |\g_i^{(0)}| - |\g_i^{(0)}x   +  \g_i^{(1)}r|  \rp^2
  -  r^2 r_y \nonumber \\
      &  = &
\alpha
\frac{r_y}{1+r_y}  f_q  -  r^2 r_y.
\label{eq:hrd7a3}
 \end{eqnarray}
Plugging this back in (\ref{eq:hrd6}) gives
\begin{eqnarray}
 \phi_0 & \triangleq & \lim_{n\rightarrow\infty} \mE_{\cG} f_{rd}(\cG)
 \nonumber \\
& \geq &
  \max_{r_y>0}
  \lp
  \alpha
\frac{r_y}{1+r_y} \mE_{\cG}\lp |\g_i^{(0)}| - |\g_i^{(0)}x   +  \g_i^{(1)}r|  \rp^2
  -  r^2 r_y \rp
  \nonumber \\
  & = &
  \max_{r_y>0}
  \lp
  \alpha
\frac{r_y}{1+r_y} f_q
  -  r^2 r_y \rp.
\label{eq:hrd7a4}
 \end{eqnarray}
It is rather obvious but we point out that inequality signs in (\ref{eq:hrd2}), (\ref{eq:hrd6}), and (\ref{eq:hrd7a4}) can be replaced with equalities. To determine optimal $r_y$ one takes the derivative of the expression under $\max$
\begin{eqnarray}
 \frac{d  \lp
  \alpha
\frac{r_y}{1+r_y} f_q
  -  r^2 r_y \rp}{dr_y} = \alpha\frac{1}{\lp1+r_y\rp^2}f_q -r^2.
\label{eq:hrd7a5}
 \end{eqnarray}
Equalling the above derivative to zero gives for the optimal $r_y$
\begin{eqnarray}
\hat{r}_y
  = \max\lp \frac{\sqrt{ \alpha f_q}}{r} -1,0\rp.
\label{eq:hrd7a6}
 \end{eqnarray}
Combining (\ref{eq:hrd7a4}) and (\ref{eq:hrd7a6}) one finally finds
\begin{eqnarray}
 \phi_0 & \triangleq & \lim_{n\rightarrow\infty} \mE_{\cG} f_{rd}(\cG)
  \geq
\max\lp \sqrt{\alpha f_q}
  -  r,0\rp^2.
\label{eq:hrd7a7}
 \end{eqnarray}
After setting
\begin{eqnarray}
A & = &  g_1^2\sqrt{c-r^2} \nonumber \\
B & = & g_1r \nonumber \\
C & = & -g_1\frac{\sqrt{c-r^2}}{r} \nonumber \\
I_1 & = & \frac{A}{2} \erfc\lp \frac{C}{\sqrt{2}}\rp   + B \frac{e^{-\frac{C^2}{2}}}{\sqrt{2\pi}} \nonumber \\
I_2 & = & -A \lp  1-\frac{1}{2}\erfc \lp  \frac{C}{\sqrt{2}}  \rp  \rp + B \frac{e^{-\frac{C^2}{2}}}{\sqrt{2\pi}} \nonumber \\
I & = & I_1+I_2,
\label{eq:hrd7a8}
\end{eqnarray}
 and solving the remaining integrals one also obtains
\begin{eqnarray}
f_{g,1} & = &  (1+c) - 2I \nonumber \\
f_q & \triangleq & \mE_{\cG}\lp |\g_i^{(0)}| - |\g_i^{(0)}x   +  \g_i^{(1)}r|  \rp^2
 =  \frac{2}{\sqrt{2\pi}} \int_{0}^{\infty}f_{q,1}e^{-\frac{g_1^2}{2}}dg_1.
 \label{eq:hrd7a9}
\end{eqnarray}

  \vspace{.1in}
\noindent \underline{4) \textbf{\emph{Double checking the strong random duality:}}}  Double checking for the strong random duality is the last step of the RDT machinery. However, since a deterministic strong duality is not present the reversal considerations from \cite{StojnicRegRndDlt10}  and the strong random duality are not in place. This basically implies that the above mechanism produces strict lower bounds on the f-pro's scaled objective, $\xi(c,x)$ .

\subsection{Algorithmic implications}
\label{sec:algimp}

The above results have very strong implications for algorithmic solving of the phase retrieval. In Figure \ref{fig:fig1} we show how $\phi_0$ changes as a function of $x$ for several different values of $c$. As can be seen from the figure, as $c$ increases the curves are lower and at the same time flatter in the low $x$ (overlap) regime. The value $\alpha^{(1)}\triangleq 1.7932$ is chosen so that for $c=1$ the curve is as flat as it can be without having a local maximum. This would indicate that as long as the sample complexity ratio $\alpha$ is larger than $\alpha^{(1)}$, (\ref{eq:ex1a4}) can be solved by \emph{any} descending algorithm (say, any form of the gradient descent) without having such an algorithm being trapped in a local optimum.
\begin{figure}[h]
\centering
\centerline{\includegraphics[width=1\linewidth]{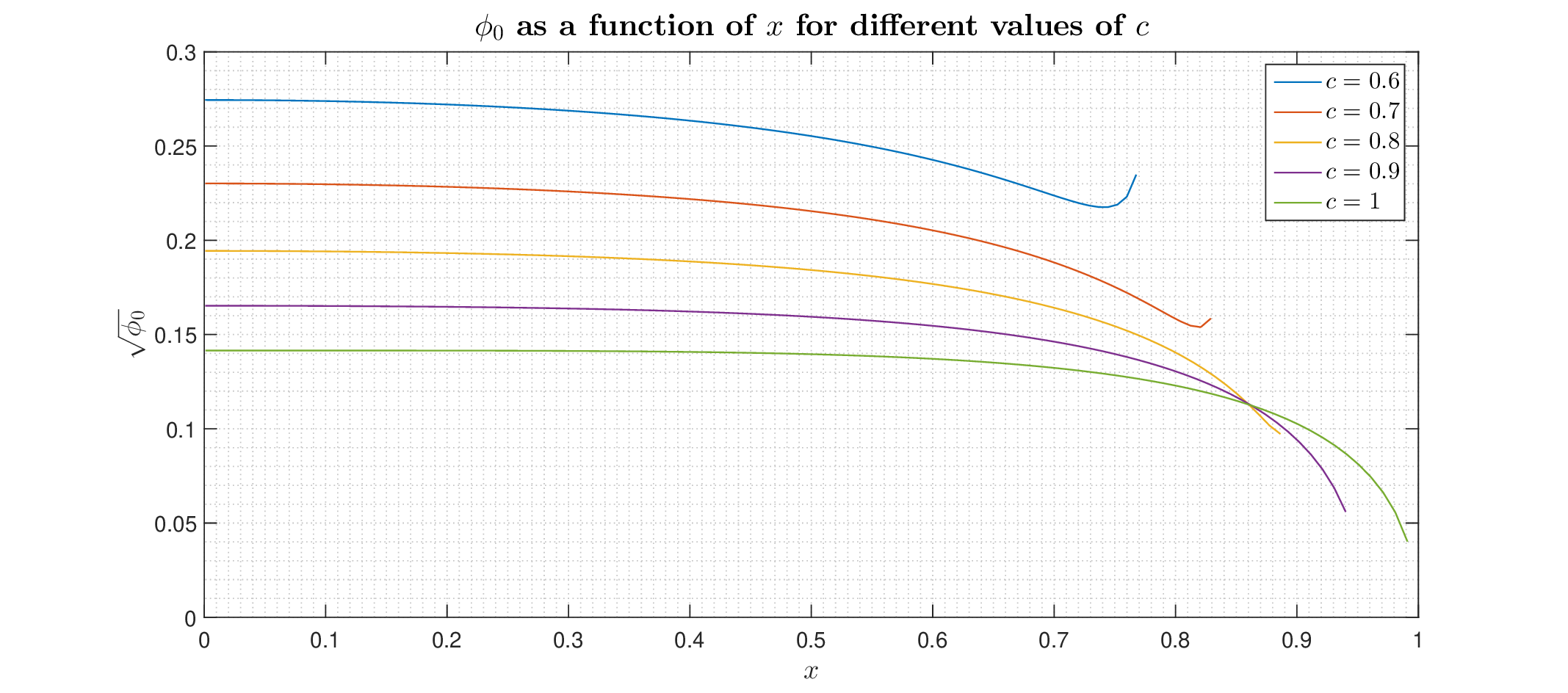}}
\caption{$\phi_0$ as a function of $x$ for different values of $c$; $\alpha=1.7932$}
\label{fig:fig1}
\end{figure}

The importance of parametric structure for both theoretical analysis and algorithmic designs was brought to prominence with the appearance of the RDT in \cite{StojnicCSetam09,StojnicICASSP10var,StojnicISIT2010binary,StojnicICASSP10block,StojnicRegRndDlt10}. To get a bit clearer picture as to how it relates to what is actually happening here and why the above indication is indeed correct, we in Figures \ref{fig:fig2}, \ref{fig:fig3}, and \ref{fig:fig4} plot the entire so-called \emph{parametric manifold}, ${\mathcal P}{\mathcal M}(\alpha)$ . This manif9old practically shows how $\phi_0$ changes as the parametric pair $(c,x)$ changes and does so for any admissible values of the pair. The shape of the manifold directly correlates to the ability of the descending algorithms to reach the global optimum in the following way: If the manifold has single ``\emph{funneling point}'' (collector of all descending paths) then any descending algorithm will converge to the global optimum. In a more pictorial informal language, no matter where one pours water over the manifold it must ultimately be collected at a \emph{single} point (in the concrete scenario of interest here, the desired funneling/collecting point is $(c,x)=(1,1)$). Figures \ref{fig:fig2}, \ref{fig:fig3}, and \ref{fig:fig4} show three different scenarios  that can develop depending  on the values of $\alpha$ (in general many more scenarios are possible as well). In the first scenario $\alpha=1.5$ and it is sufficiently small to allow for emergence of an alternative funneling point at $(c,x)=(1,0)$. For $\alpha=2.3$ there is only one funneling point $(c,x)=(1,1)$ whereas for $\alpha=1.7932$ we have a limiting scenario below which one starts observing appearance of the second funneling point. Manifold plots in  Figures \ref{fig:fig2}, \ref{fig:fig3}, and \ref{fig:fig4} were shown as functions of $(c,x)$ pairs. Alternatively, one can utilize scaled variant of overlap and instead of $x$ plot $x/c$. The manifolds would shrink or get stretched out but no significant conceptual structural changes would be seen (what actually determines the critical manifold properties is related to behavior for any fixed $c$ and scaling by $c$ is irrelevant in that aspect).


\begin{figure}[h]
\centering
\centerline{\includegraphics[width=1\linewidth]{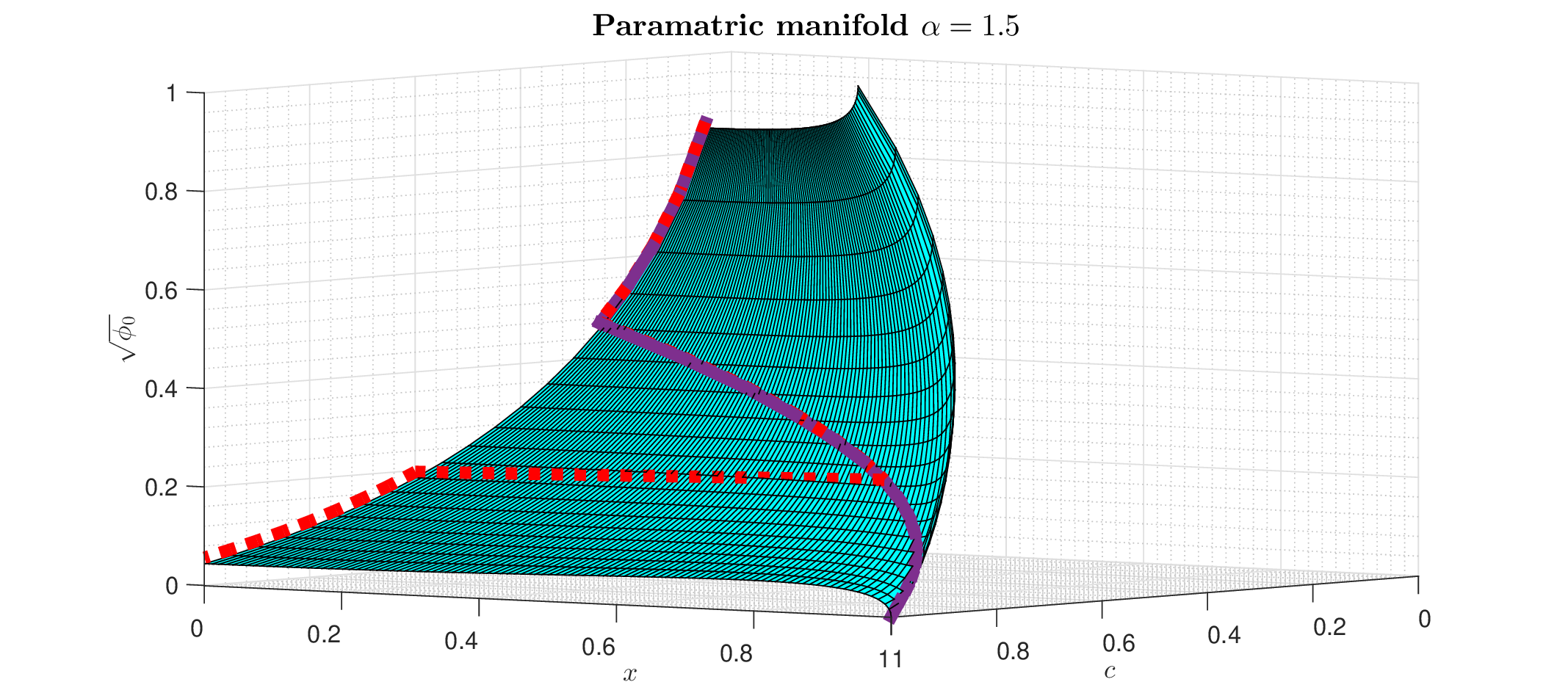}}
\caption{Parametric manifold for $\alpha=1.5$, ${\mathcal P}{\mathcal M}(1.5)$; Red/purple curves -- undesired/desired funneling flows}
\label{fig:fig2}
\end{figure}
The above paints a theoretical and somewhat idealized picture. In practical scenarios many other factors need to align so that the above logic can indeed be of use. \textbf{\emph{(i)}} First, everything discussed above heavily relies on concentrations and effectively assumes that in any practical realization the manifold is indeed as smooth as figures show. For $n\rightarrow\infty$ this is indeed the case with probability go9ng to 1. As $n$ can not be unlimited in practical scenarios one has to be mindful that the above logic might on occasion fail. This is particularly likely to happen if one is close to the limiting scenario (here that would mean if sample complexity ratio is close to $1.7932$). Consequently, a good strategy for practical running of the descending algorithms is to slightly oversample to avoid potential effects of undesired lack of concentrations, \textbf{\emph{(ii)}}The above also assumes that one has a mechanism of controlling $c$ so that it remains below $1$. There are many techniques that can do that but one has to be careful how to implement them to ensure that algorithms avoid getting trapped on the boundaries. \textbf{\emph{(iii)}} Finally as the strong RDT is not in place one needs to keep in mind that the estimated oversampling is likely to go down if one implements lifted RDT (this does not affect conceptual conclusions mentioned earlier, but may have an effect on the concrete value of the minimal needed oversampling ratio). We discuss this in detail in the following section and show  that the critical oversampling value can indeed be lowered.

\begin{figure}[h]
\centering
\centerline{\includegraphics[width=1\linewidth]{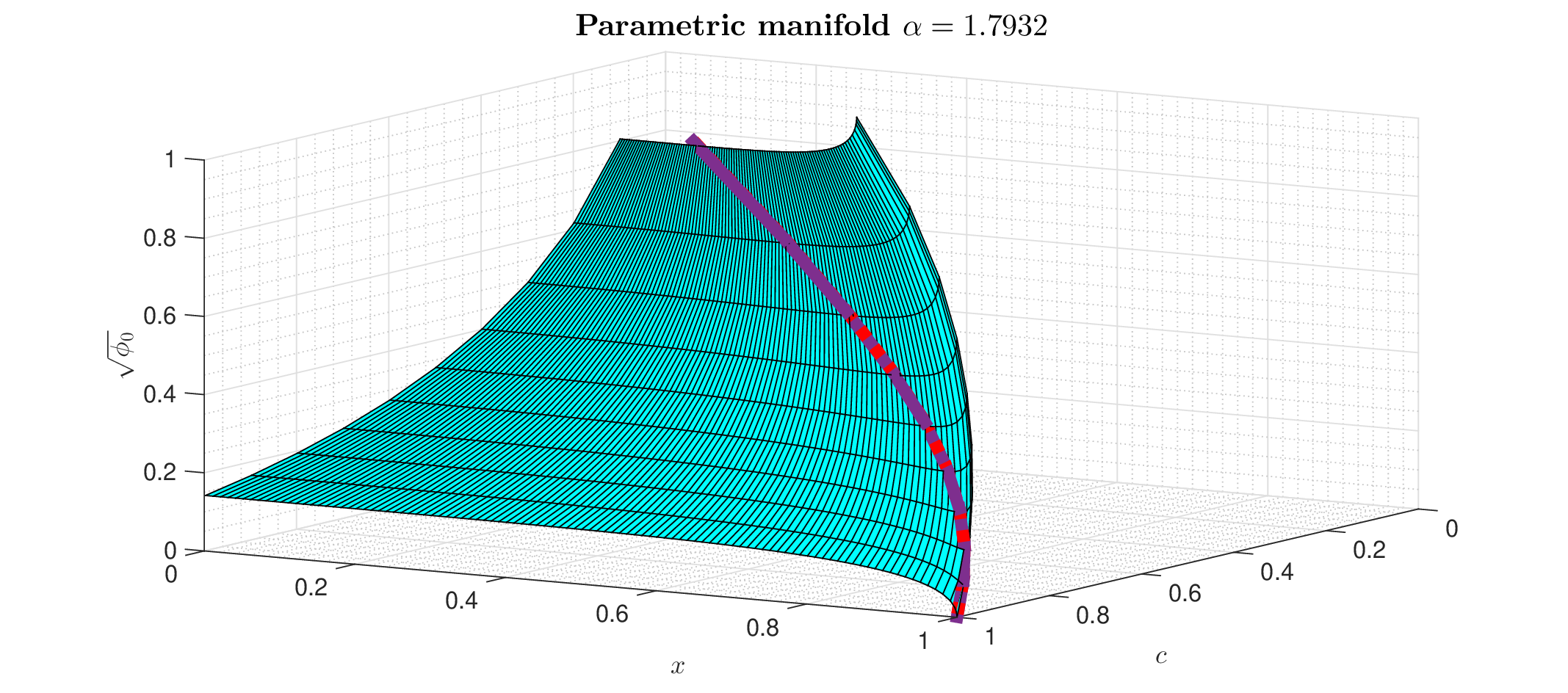}}
\caption{Parametric manifold for $\alpha=1.7932$; Red/purple curves -- undesired/desired funneling flows}
\label{fig:fig3}
\end{figure}

\begin{figure}[h]
\centering
\centerline{\includegraphics[width=1\linewidth]{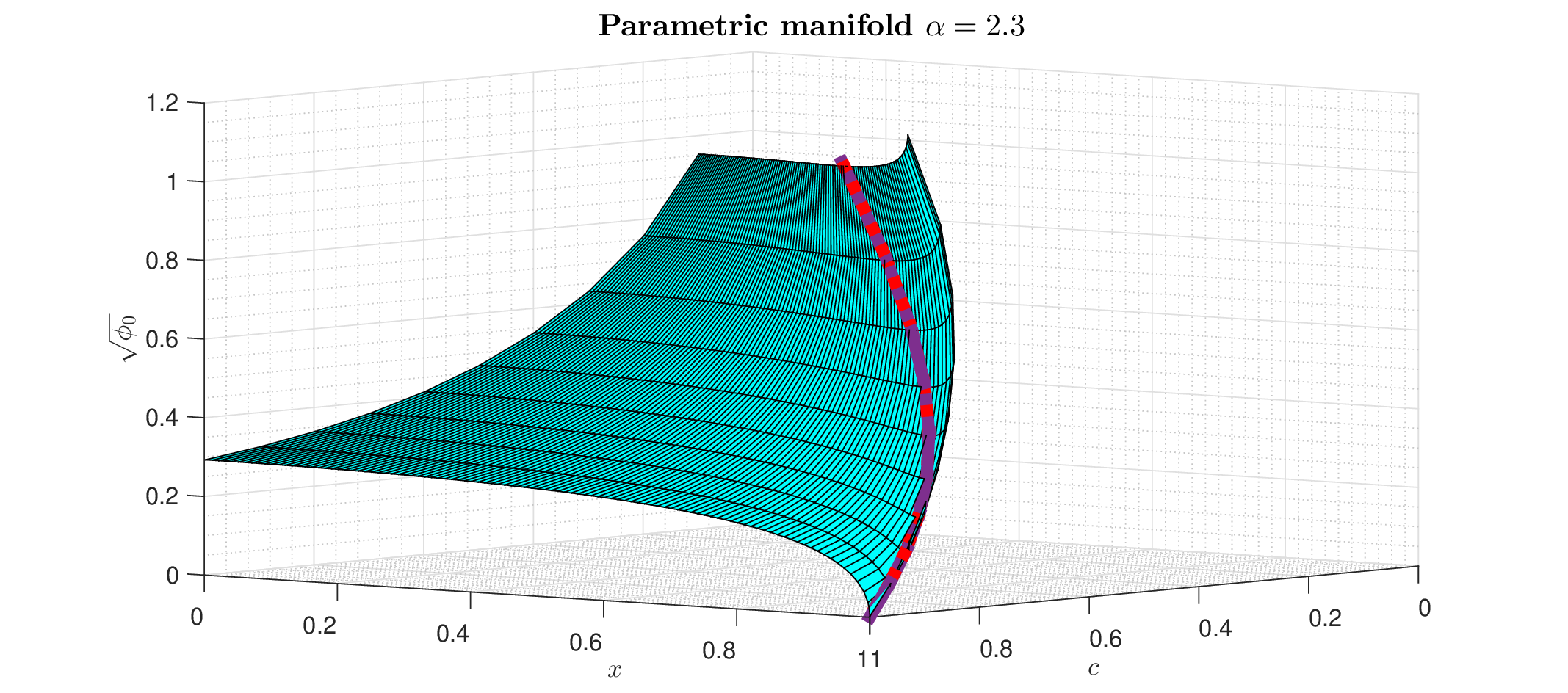}}
\caption{Parametric manifold for $\alpha=2.3$; Red/purple curves -- undesired/desired funneling flows}
\label{fig:fig4}
\end{figure}

\subsection{Beyond optimal objective landscape}
 \label{secbeyondland}

The above discussed algorithmic implications are driven by the structure of the optimal objective related parametric manifold. The intuition suggests that given that fl RDT is likely needed to exactly determine all associated quantities, the above manifold study can indeed be among the key driving forces behind success of practically feasible algorithms. However, many other intrinsic features known or unknown to exist within random structures may have an impact as well.

As the theoretical limit of the sample complexity for the scenario studied here is $1$, one can clearly observe that the above discussion indicates a potential existence of the so-called computational gap (C-gap). C-gap is a performance mismatch between practically feasible and computationally unrestrictive algorithms and is observed in many well known random optimization problems \cite{MMZ05,GamarSud14,GamarSud17,GamarSud17a,AchlioptasR06,AchlioptasCR11,GamMZ22}. While the algorithms that we consider here are in no way generic representatives of the best practically feasible ones, they do come form a class of fairly successful ones that often do stand among practically best ones.

The source of  C-gaps has been studied to a great extent  over the last decade. We here briefly mention two particular lines of work that have in a way dominated recent literature: (i) the first one relates to the so-called \emph{Overlap gap property} (OGP) \cite{Gamar21,GamarSud14,GamarSud17,GamarSud17a,AchlioptasCR11,HMMZ08,MMZ05}; and (ii) the second one to the so-called \emph{Local entropy} (LE) \cite{Bald15,Bald16,Bald20}. As we will see below both properties attempt to demystify clustering organizations of optimal or ``near'' optimal solutions.

The main idea behind the OGP approach \cite{Gamar21,GamarSud14,GamarSud17,GamarSud17a,AchlioptasCR11,HMMZ08,MMZ05} is to connect algorithmic efficacy with a lack of gaps in attainable (near) optimal solutions pairs Hamming distances. For several random graph problems it is known that the transition from presence to absence of OGP precisely matches the transition from presumed algorithmic hardness to known algorithmic solvability \cite{Gamar21,GamarSud14,GamarSud17,GamarSud17a}. Recent OGP studies on symmetric binary perceptrons (SBP) from  \cite{GamKizPerXu22,BarbAKZ23} together with corresponding algorithmic considerations from \cite{BanSpen20} indicate similar OGP role. While the appearance of \cite{LiSch24} disproves OGP generic hardness implication in a non-trivial nature (earlier trivial disproving examples were usually viewed as exceptions), its relevance for particular algorithms or other optimization problems remains. It is sufficient just to recall on breakthrough results of \cite{Montanari19} where solvability in polynomial time of the famous (SK) model \cite{SheKir72} is shown to be directly implied by the widely believed absence of OGP
(further results on p-spin SK extensions \cite{AlaouiMS22,AlaouiMS21} and more sophisticated OGPs  \cite{Kiz23,HuangS22} are available as well; moreover, related spherical spin-glass models results were obtained earlier in  \cite{Subag17,Subag17a,Subag21,Subag24}). Despite its role in generic algorithmic hardness being undetermined, the OGP's presence certainly precludes practically efficient specific  algorithmic classes \cite{GamKizPerXu22}.

Studying clustering organization and associated algorithmic hardness implications via entropies is an alternative to the above OGP. Focusing on binary peceptrons,  \cite{Huang13,Huang14}  looked at the \emph{typical} solutions entropy and predicated via replica methods that such solutions are in complete so-called \emph{frozen} isolation (\cite{PerkXu21,AbbLiSly21a,AbbLiSly21b} proved these predictions for SBP). \cite{Bald15,Bald16,Bald20} considered a stronger entropic refinement and proposed studying  \emph{local entropy} (LE) of \emph{atypical} well connected clusters. While predominant typical solutions might be completely disconnected  and unreachable though local searches \cite{Huang13,Huang14,PerkXu21,AbbLiSly21b}, there may still exist  well-connected but rare (subdominant) clusters.  Moreover, \cite{Bald15,Bald16,Bald20}  then predicted that precisely such rare clusters are found by efficient algorithms (a sampling type of justification for SBP can be found in,  e.g., \cite{ElAlGam24}). Provided that his phenomenology is indeed in place, the properties of rare clusters directly impact the C-gap existence. \cite{Bald15,Bald16,Bald20} additionally  speculate that particular LE's proeprties (breakdown, monotonicity, negativity, etc.) play a key role. For a further support of such a view see, e.g.,\cite{AbbLiSly21a,BarbAKZ23,Stojnicabple25}. Keeping all of this in mind, studying  potential impact of both OGP and LE on phase retrieval problems/algorithms might be of great interest. In particular, adapting the machinery introduced in \cite{Stojnicnflldp25,Stojnicsflldp25,Stojnicabple25} so that it accommodates such studies is the first next step.

\section{Lifted RDT}
 \label{sec:liftrdt}

 Since the strong random duality is not in place the above RDT bounds are expected to be \emph{strict} which then implies that  they can be further lifted. With the development of \emph{fully lifted} (fl) RDT \cite{Stojnicflrdt23} a precise characterization of the optimal values is possible. However, to fully implement the fl RDT a sizeable set of numerical evaluations needs to be conducted. Instead of doing so, we consider a bit more convenient  \emph{partially lifted} (pl) RDT variant \cite{StojnicLiftStrSec13,StojnicMoreSophHopBnds10,Stojnicinjdeeprelu24}. Although numerically less intensive, pl RDT is able to provide a strong improvement over plain RDT as well.

As discussed on many occasions in  \cite{Stojnicinjdeeprelu24,StojnicLiftStrSec13,StojnicMoreSophHopBnds10}, the pl RDT relies on similar principles as the plain RDT. A notable exception though is that now one considers \emph{partially lifted random dual}. The following theorem  (a partially lifted analogue to Theorem \ref{thm:thm1}) introduces it.
\begin{theorem} Assume the setup of Theorem \ref{thm:thm1} with the elements of $A\in\mR^{m\times n}$ ($\g^{(0)}\in\mR^{m\times 1}$ and $A_{:,2:n}\in\mR^{m\times (n-1)}$), $\g^{(1)}\in\mR^{m\times 1}$, and  $\h^{(1)}\in\mR^{(n-1)\times 1}$  being iid standard normals. Consider two positive scalars $c$ and $x$  ($0\leq x \leq c$) and set $r\triangleq \sqrt{c-x^2}$. Let $c_3>0$ and
\vspace{-.0in}
\begin{eqnarray}
\cG & \triangleq & \lp A,\g^{(1)},\h^{(1)}\rp = \lp\g^{(0)},A_{:,2:n},\g^{(1)},\h^{(1)}\rp  \nonumber \\
\phi(\x,\z,\y) & \triangleq &
 \lp \||\g^{(0)}|-|\z| \|_2^2
  +\y^T \g^{(0)}x   +  \y^T \g^{(1)}\|\x_{2:n}\|_2 + \lp \x_{2:n}  \rp^T\h^{(1)}\|\y\|_2  -\y^T  \z  \rp
\nonumber \\
 \bar{f}_{rd}(c,x;\cG) & \triangleq &
 \min_{\|\x_{2:n}\|_2=r,\z} \max_{\|\y\|_2=r_y}  \phi(\x,\z,\y)
  \nonumber \\
  \bar{\phi}_0 & \triangleq & \max_{r_y>0}\lim_{n\rightarrow\infty} \frac{1}{n}
 \lp
 \frac{c_3}{2} r^2r_y^2 -
\frac{1}{c_3} \log \lp \mE_{\cG_{(2)}} e^{ - c_3 \bar{f}_{rd}(\cG) } \rp   \rp .\label{eq:plta16}
\vspace{-.0in}\end{eqnarray}
One then has \vspace{-.02in}
\begin{eqnarray}
  \lim_{n\rightarrow\infty}\mP_{ A } (f_{rp} (c,x; A )   > \bar{\phi}_0)\longrightarrow 1,\label{eq:plta17a0}
\end{eqnarray}
and
\begin{eqnarray}
\hspace{-.0in}(\bar{\phi}_0  > 0)    \Longrightarrow \lp \lim_{n\rightarrow\infty}\mP_{ A } (f_{rp} (c,x; A )   >0)\longrightarrow 1 \rp
 \Longrightarrow \lp \lim_{n\rightarrow\infty}\mP_{A} \lp \mbox{PR is (uniquely) solvable} \rp \longrightarrow 1\rp.\label{eq:plta17}
\end{eqnarray}
 \label{thm:thm2}
\end{theorem}\vspace{-.17in}
\begin{proof}
For any  fixed $r_y$ it follows as an automatic application of Corollary 3 from  \cite{Stojnicgscompyx16} (see in particular Section 3.2.1 and equation (86) as well as  Lemma 2 and equation (57) in \cite{StojnicMoreSophHopBnds10}). Terms  $ \y^T \g^{(1)}\|\x_{2:n}\|_2$,  $\x_{2:n}^T\h^{(1)}\|\y\|_2$,  $\y^T  \z  $, and $\frac{c_3}{2}r^2r_y^2$ correspond to the lower-bounding side of equation (86) in \cite{Stojnicgscompyx16} and the left hand side of (86) corresponds to $f_{rp}$. Concentrations and maximization over $r_y$ complete the proof.
\end{proof}

Handling of the above partially lifted random dual can be substantially sped up if one relies on the results from previous sections. To that end we first note that, analogously to (\ref{eq:hrd1}) and (\ref{eq:hrd2}), one can write
\begin{eqnarray}
\bar{f}_{rd}(c,x;\cG)
& = &
 \min_{\z}
 \lp \||\g^{(0)}|-|\z| \|_2^2
  +\|\g^{(0)}x   +  \g^{(1)}r -\z\|_2 r_y - \|\h^{(1)} \|_2 r r_y \rp.
\label{eq:plhrd1}
\end{eqnarray}
Utilizing the \emph{square root trick} introduced on numerous occasions in, e.g., \cite{StojnicLiftStrSec13}  we then further have
\begin{eqnarray}
\bar{f}_{rd}(c,x;\cG)
& = &
\min_{\z}
 \lp \||\g^{(0)}|-|\z| \|_2^2
  +\|\g^{(0)}x   +  \g^{(1)}r -\z\|_2 r_y - \|\h^{(1)} \|_2 r r_y \rp.
\nonumber \\
& = &
\min_{\z,\gamma>0}\max_{\gamma_{sph}>0}
 \lp \||\g^{(0)}|-|\z| \|_2^2
+\gamma   +\frac{\|\g^{(0)}x   +  \g^{(1)}r -\z\|_2^2 r_y^2}{4\gamma} - \gamma_{sph} - \frac{\|\h^{(1)} \|_2^2 r^2 r_y^2 }{4\gamma_{sph}}   \rp.
\nonumber \\
& = &
\min_{\z,\gamma>0}\max_{\gamma_{sph}>0}
 \bar{\cL}(r_y)
 \label{eq:plhrd2}
 \end{eqnarray}
where
\begin{eqnarray}
\bar{\cL}(r_y) = \lp  \sum_{i=1}^{m}\||\g_i^{(0)}|-|\z_i| \|_2^2
+\gamma   + \frac{\sum_{i=1}^{m}\|\g^{(0)}x   +  \g^{(1)}r -\z\|_2^2 r_y^2}{4\gamma} - \gamma_{sph} - \frac{\sum_{i=1}^{n}\|\h^{(1)} \|_2^2 r^2 r_y^2 }{4\gamma_{sph}}   \rp.
 \label{eq:plhrd3}
 \end{eqnarray}
After setting
\begin{eqnarray}
\bar{r}_y=\frac{r_y^2}{4\gamma},
 \label{eq:plhrd4}
\end{eqnarray}
one can optimize over $\z_i$ as in the previous section to obtain analogously to (\ref{eq:hrd7a1}) and (\ref{eq:hrd7a2})
\begin{eqnarray}
|\hat{\z}_i|= \frac{1}{1+\bar{r}_y} \lp |\g_i^{(0)}| + |\g_i^{(0)}x   +  \g_i^{(1)}r| \bar{r}_y \rp.
\label{eq:plhrd5}
 \end{eqnarray}
 and
 \begin{eqnarray}
\min_{\z} \bar{\cL}(r_y)
& = &
\min_{\z}  \lp  \sum_{i=1}^{m}\||\g_i^{(0)}|-|\z_i| \|_2^2
+\gamma   + \frac{\sum_{i=1}^{m}\|\g^{(0)}x   +  \g^{(1)}r -\z\|_2^2 r_y^2}{4\gamma} - \gamma_{sph} - \frac{\sum_{i=1}^{n}\|\h^{(1)} \|_2^2 r^2 r_y^2 }{4\gamma_{sph}}   \rp
\nonumber \\
& = &
  \lp  \sum_{i=1}^{m} \frac{\bar{r}_y}{1+\bar{r}_y} \lp |\g_i^{(0)}| - |\g_i^{(0)}x   +  \g_i^{(1)}r|  \rp^2
+\gamma   - \gamma_{sph} - \frac{\sum_{i=1}^{n}\|\h^{(1)} \|_2^2 r^2 r_y^2 }{4\gamma_{sph}}   \rp.
\label{eq:plhrd6}
 \end{eqnarray}
A combination of  (\ref{eq:plhrd2}) and (\ref{eq:plhrd6}) then gives
\begin{eqnarray}
\bar{f}_{rd}(c,x;\cG)
 & = &
\min_{\z,\gamma>0}\max_{\gamma_{sph}>0}
 \bar{\cL}(r_y)
  \nonumber \\
 & = &
\min_{\gamma>0}\max_{\gamma_{sph}>0}
  \lp  \sum_{i=1}^{m} \frac{\bar{r}_y}{1+\bar{r}_y} \lp |\g_i^{(0)}| - |\g_i^{(0)}x   +  \g_i^{(1)}r|  \rp^2
+\gamma   - \gamma_{sph} - \frac{\sum_{i=1}^{n}\|\h^{(1)} \|_2^2 r^2 r_y^2 }{4\gamma_{sph}}   \rp.
\nonumber \\
  \label{eq:plhrd7}
 \end{eqnarray}
After appropriate scaling   $\gamma \rightarrow \gamma n$, $\gamma_{sph}\rightarrow \gamma_{sph} n $, $r_y\rightarrow r_y\sqrt{n}$, concentrations, statistical identicalness over $i$, Lagrangian duality, and a combination of (\ref{eq:plta16}) and (\ref{eq:plhrd7}) give
\begin{eqnarray}
  \bar{\phi}_0 & \triangleq & \max_{r_y>0}\lim_{n\rightarrow\infty} \frac{1}{n}
 \lp
 \frac{c_3}{2} r^2r_y^2 -
\frac{1}{c_3} \log \lp \mE_{\cG_{(2)}} e^{ - c_3 \bar{f}_{rd}(\cG) } \rp   \rp
 \nonumber \\
& \geq &
\max_{r_y>0}\min_{\gamma>0} \max_{\gamma_{sph}>0}
\lp
 \frac{c_3}{2} r^2r_y^2 + \gamma
 -\frac{\alpha}{c_3} \log \lp \mE_{\cG} e^{ - c_3 \bar{f}_{q}  } \rp
 - \gamma_{sph}  - \frac{1}{c_3} \log \lp \mE_{\cG_{(2)}} e^{ c_3 \frac{\lp \h_i^{(1)}\rp^2r^2r_y^2}{4\gamma_{sph}} }\rp
\rp,
\label{eq:plhrd8}
 \end{eqnarray}
where
 \begin{eqnarray}
 \gamma_x & = & \frac{\bar{r}_y}{1+\bar{r}_y}
 \nonumber \\
 \bar{f}_{q} & = &
\gamma_x \lp |\g_i^{(0)}| - |\g_i^{(0)}x   +  \g_i^{(1)}r|  \rp^2.
 \label{eq:plhrd9}
 \end{eqnarray}
After setting
\begin{eqnarray}\label{eq:plhrd10}
    \bar{A} & = & -2c_3\gamma_x r |g_1| + 2c_3\gamma_x g_1 r \sqrt{c-r^2} \nonumber \\
    \bar{B} & = & -\frac{g_1\sqrt{c-r^2}}{r}   \nonumber \\
    \bar{C} & = & 1 + 2c_3\gamma_x r^2   \nonumber \\
    \bar{D} & = & -c_3\gamma_x  ( g_1^2(1+c-r^2) -2g_1 |g_1| \sqrt{c-r^2}    )     \nonumber \\
    \bar{I}_1 & = &  \frac{e^{\bar{D} +    \frac{\bar{A}.^2}{2\bar{C}}   }   }{2\sqrt{\bar{C}}} \erfc\lp \frac{ \bar{A}+\bar{C}\bar{B} } {\sqrt{2\bar{C}} }  \rp  \nonumber \\
    \bar{A}_2 & = & 2c_3\gamma_x r |g_1| + 2c_3\gamma_x g_1 r \sqrt{c-r^2} \nonumber \\
    \bar{B}_2 & = & -\frac{g_1\sqrt{c-r^2}}{r}   \nonumber \\
    \bar{C}_2 & = & 1 + 2c_3\gamma_x r^2   \nonumber \\
    \bar{D}_2 & = & -c_3\gamma_x  ( g_1^2(1+c-r^2) + 2g_1 |g_1| \sqrt{c-r^2}    )  \nonumber \\
    \bar{I}_2 & = &   \frac{e^{\bar{D}_2  +    \frac{\bar{A}_2.^2}{2\bar{C}_2}   }   }{2\sqrt{\bar{C}_2}} \erfc \lp -\frac{ \bar{A}_2+\bar{C}_2\bar{B}_2 } {\sqrt{2\bar{C}_2} }  \rp ,
 \end{eqnarray}
and solving the integrals one obtains
\begin{eqnarray}
f_{q}^{(lift)}  & = &  \mE_{\cG} e^{-c_3 \bar{f}_{q}}
  =
\int_{-\infty}^{\infty}
( \bar{I}_1 +  \bar{I}_2) \frac{e^{-\frac{ \lp g_1 \rp^2     } {2}  }}{\sqrt{2\pi}} dg_1.
  \label{eq:plhrd11}
 \end{eqnarray}
One then sets
\begin{eqnarray}
c_{3,e}=c_3r_yr,
  \label{eq:plhrd12}
\end{eqnarray}
solves the integral over $\h_i^{(1)}$ and optimizes over $\gamma_{sph}$ to obtain (see, e.g.,  \cite{StojnicMoreSophHopBnds10,Stojnicinjdeeprelu24})
\begin{eqnarray}
\hat{\gamma}_{sph} =\frac{c_{3,e}+\sqrt{c_{3,e}^2+4}}{4}=\frac{c_3rr_y+\sqrt{c_3^2r^2r_y^2+4}}{4}.
  \label{eq:plhrd13}
 \end{eqnarray}
Connecting (\ref{eq:plhrd8}) with (\ref{eq:plhrd11})-(\ref{eq:plhrd13}) gives
\begin{eqnarray}
   \bar{\phi}_0
& \geq &
\max_{r_y>0}\min_{\gamma>0} \max_{\gamma_{sph}>0}
\lp
 \frac{c_3}{2} r^2r_y^2 + \gamma
 -\frac{\alpha}{c_3} \log \lp \mE_{\cG} e^{ - c_3 \bar{f}_{q}  } \rp
 - \gamma_{sph}  - \frac{1}{c_3} \log \lp \mE_{\cG_{(2)}} e^{ c_3 \frac{\lp \h_i^{(1)}\rp^2r^2r_y^2}{4\gamma_{sph}} }\rp
\rp
 \nonumber \\
 & = &
\max_{r_y>0}\min_{\gamma>0}
\Bigg .\Bigg(
 \frac{c_3}{2} r^2r_y^2 + \gamma
 -\frac{\alpha}{c_3} \log \lp f_{q}^{(lift)}\rp
 - \hat{\gamma}_{sph}  +\frac{1}{2c_3} \log \lp  1  -  \frac{c_{3,e}}{2\hat{\gamma}_{sph}}     \rp
\Bigg.\Bigg),
\label{eq:plhrd14}
 \end{eqnarray}
with $f_{q}^{(lift)}$ and $\hat{\gamma}_{sph}$  as in    (\ref{eq:plhrd11})) and  (\ref{eq:plhrd13}), respectively. Since (\ref{eq:plhrd14}) holds for any $c_3$, maximizing the right hand side over $c_3$ further gives
\begin{eqnarray}
 \bar{\phi}_0
& \geq &
 \max_{c_3> 0}  \max_{r_y>0}\min_{\gamma>0}
\Bigg .\Bigg(
 \frac{c_3}{2} r^2r_y^2 + \gamma
 -\frac{\alpha}{c_3} \log \lp f_{q}^{(lift)}\rp
 - \hat{\gamma}_{sph}  +\frac{1}{2c_3} \log \lp  1  -  \frac{c_3rr_y  }{2\hat{\gamma}_{sph}}     \rp
\Bigg.\Bigg).
\label{eq:plhrd16a0}
 \end{eqnarray}

It is not that relevant but we note that inequality signs in (\ref{eq:plhrd8}), (\ref{eq:plhrd14}), and (\ref{eq:plhrd16a0}) can be replaced with equalities. The effect of the above lifting mechanism is presented in Figure \ref{fig:fig5}. For larger values of $c$ (closer to $c=1$) the improvement through lifted RDT is more pronounced.  As $c$ decreases the improvement decreases as well and already for $c=0.8$ it is visually undetectable. Generic conclusions reached earlier when discussing similar plots for plain RDT continue to hold. Namely,, as $c$ increases the curves are lower and flatter for smaller overlaps $x$. The value $\alpha^{(2,par)}\triangleq 1.4$ is chosen so that $c=1$ curve is nonincreasing in any subinterval of $x\in[0,\sqrt{c}]$. As earlier, this basically indicates that as long as sample complexity ratio $\alpha$ is larger than $\alpha^{(2,par)}$, (\ref{eq:ex1a4}) can be solved by \emph{any} descending algorithm. This effectively means that the needed oversampling for success of the phase retrieval descending algorithms is  dropped from $\alpha^{(1)}\approx 1.7932$  to $\alpha^{(2,par)}\approx 1.4$ through the lifting mechanism.
\begin{figure}[h]
\centering
\centerline{\includegraphics[width=1\linewidth]{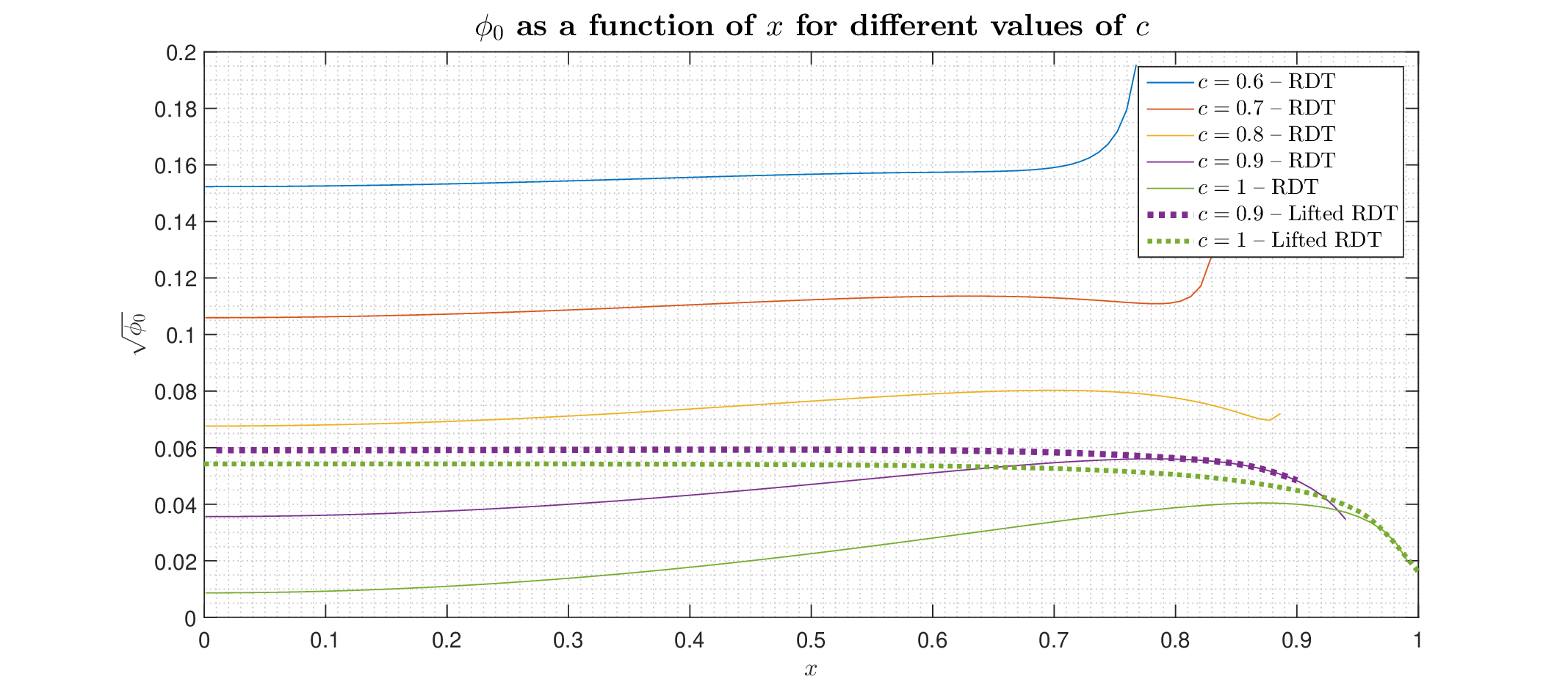}}
\caption{Effect of \emph{lifted} RDT -- $\phi_0$ as a function of $x$ for different values of $c$; $\alpha=1.4$}
\label{fig:fig5}
\end{figure}

Analogously to  Figures \ref{fig:fig2}, \ref{fig:fig3}, and \ref{fig:fig4}, Figure  \ref{fig:fig6} presents the entire \emph{parametric manifold} for $\alpha=1.4$. Both plain RDT and partially lifted RDT  (pl RDT) manifold characterizations are shown in parallel. One can clearly see that pl RDT provides a rather substantial lifting. Most notably, instead of two ``funneling points'', $(c,x)=(1,0)$ and $(c,x)=(1,1)$, manifold after lifting has only one funneling point $(c,x)=(1,1)$. As  discussed earlier, if the manifold has single (desired) ``\emph{funneling point}'' (collector of all descending paths) then any descending algorithm converges to the global optimum. Speaking in a pictorial language, if one pours water on the lower manifold it might funnel down to both  collecting points $(c,x)=(1,0)$ and $(c,x)=(1,1)$. This practically means that if one starts a descending phase retrieval algorithm at an unfavorable part of the manifold, then it might not converge to the desired solution (the one with magnitude $c=1$ and overlap $x=\x^T\bar{\x}=1$). On the other hand, water poured on the upper manifold will always be collected at $(c,x)=(1,1)$, which translates to the fact that any descending algorithm will converge to the desired solution no matter where on the manifold the algorithm starts.  The choice $\alpha\approx1.4$ is  particularly tailored and corresponds to the smallest needed oversampling which in the context of the pl RDT ensures that parametric manifold remains with a single funneling point. Further decrease of $\alpha$ would cause even lifted manifold to have multiple funneling points which would imply a generic failure of the descending phase retrieval algorithms.

\begin{figure}[h]
\centering
\centerline{\includegraphics[width=1\linewidth]{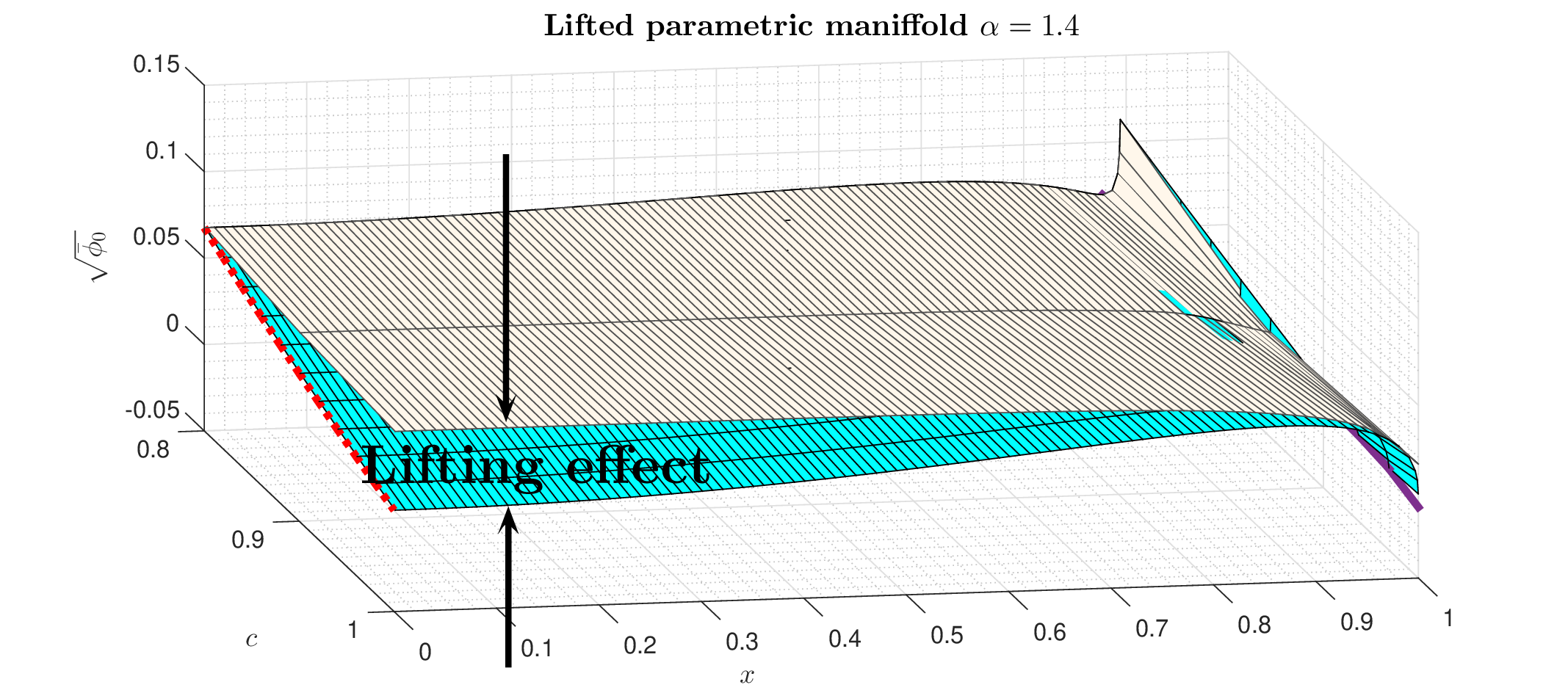}}
\caption{Lifted parametric manifold for $\alpha=1.4$}
\label{fig:fig6}
\end{figure}

All points made earlier regarding the above theoretical picture being a bit idealized remain in place. In particular, one has to be careful that the smoothness of the manifold assumes concentrations which in theory require $n\rightarrow\infty$. In practical  reality, concrete manifold realizations are likely to have \emph{global} tendencies similar to those shown in figures but are also highly unlikely to be \emph{locally} smooth. Instead, one can expect for any finite $n$ \emph{locally jittery} behavior which may cause occasional algorithmic traps. We discuss practical implementations next.

\section{Practical implementations}
 \label{sec:pract}

As discussed earlier, in practical algorithmic implementations, one has to be careful about two key potential obstacles: \textbf{\emph{(i)}} The parametric manifold can be locally jittery causing descending algorithms to get trapped; and \textbf{\emph{(ii)}} Theoretical considerations presented above assume that the norm of the targeted vector, $\sqrt{c}$, remains below $1$ which forces a constrained optimization. This second point makes direct use of the simple (unconstrained) gradient method (similar to Wirtinger flow) somewhat inconvenient. We instead tested a  log barrier version of the gradient. The objective we optimize is the following
\begin{eqnarray}
f_{bar}(t_0;\x) \triangleq  t_0\||A\bar{\x}|^2-|A\x|^2 \|_2^2 + \log\lp 1-\|\x\|_2^2 \rp
=t_0f_{plain}(\x) + \log\lp 1-\|\x\|_2^2 \rp,
\label{eq:practeq1}
\end{eqnarray}
where
\begin{eqnarray}
f_{plain}(\x) \triangleq \||A\bar{\x}|^2-|A\x|^2 \|_2^2
\label{eq:practeq1a0}
\end{eqnarray}
One should note that $f_{plain}(\x)$ is slightly different from the version analyzed earlier as we here utilize (derivative) smoother squared magnitudes  rather than just magnitudes (the analysis presented earlier remains conceptually unaltered if one uses squared magnitude in (\ref{eq:ex1a4a0}); it is just that some optimizations do not admit elegant closed form solutions and one instead has to resort to additional numerical evaluations; see Sections \ref{sec:sqadj} and \ref{sec:sqliftrdt}). Also, we take norm to be smaller than 1 which in a way assumes a prior knowledge of the norm. However, such assumption is not necessary. Rerunning our procedure   $\sim 1/\epsilon$ times (which does not change the complexity order) with $\sim 1/\epsilon$ different norms allows to obtain the same type of results that we present below even without a priori knowing the norm.

We apply optimization procedure $\mathbf{gradback}$ -- gradient with a backtracking to ensure norm constraints are satisfied -- on $f_{bar}(t_0;\x^{(gb,0)}) $ (where $\x^{(gb,0)}$ is a starting point)  and denote output as $\x^{(gb)}$. We then iteratively repeat it for an increasing schedule of $t_0$ (until $t_0$ is sufficiently large, say $10^7$)
\begin{eqnarray}
\bl{\mathbf{gradbar:}} \qquad   \x^{(gb,i+1)}&  =  & \mathbf{gradback}(f_{bar}(t_0^{(i)};\x^{(gb,i)})) \qquad \mbox{and} \qquad  t_0^{(i+1)}=1.2t_0^{(i)}, i=0,1,2,\dots.
\label{eq:practeq2}
\end{eqnarray}
We typically  rely on a spectral initialization for $\x^{(gb,0))}$ and take  $t_0^{(0)}=0.00005$. Also, we refer to the above mechanism as $\mathbf{gradbar}$ and for input $\a^{(in)}$ denote its output as $\a^{(out)}$
\begin{eqnarray}
\a^{(out)} &  =  & \mathbf{gradbar} \lp  \a^{(in)} \rp.
\label{eq:practeq3}
\end{eqnarray}
This handles the constrained optimization potential obstacle mentioned above.

To handle potential local jitteriness we utilize a hybrid alternation between $\mathbf{gradbar}$ and plain gradient (without the norm constraint). Analogously to (\ref{eq:practeq2}), we first define plain gradient as
\begin{eqnarray}
\bl{\mathbf{gradplain:}} \qquad   \x^{(gb,1)}&  =  & \mathbf{gradback}(f_{plain}(\x^{(gb,0)})).
\label{eq:practeq4}
\end{eqnarray}
Moreover, to be in a semantic agreement with (\ref{eq:practeq3}), we refer to the above plain gradient mechanism from (\ref{eq:practeq4}) as $\mathbf{gradplain}$ and for input $\a^{(in)}$ denote its output as $\a^{(out)}$
\begin{eqnarray}
\a^{(out)} &  =  & \mathbf{gradplain}(\a^{(in)}).
\label{eq:practeq5}
\end{eqnarray}
We then have for the $\mathbf{hybrid}$ the following alternating procedure
\begin{eqnarray}
\bl{\mathbf{hybrid:}} \qquad    \a^{(i+1)} &  =  & \mathbf{reshuffle} \lp \mathbf{gradplain} \lp \mathbf{reshuffle} \lp\mathbf{gradbar}\lp  \a^{(i)} \rp \rp \rp \rp, i=0,1,2,\dots,
\label{eq:practeq6}
\end{eqnarray}
where $\mathbf{reshuffle}$ takes a fraction (say $5-10\%$) of components a few times (rarely did we need more than 10 times) and changes their signs. As mentioned above, we take spectral initialization for $\a^{(0)}$, i.e. we take
\begin{eqnarray}
\bl{\mbox{\emph{\textbf{spectral initialization:}}}} \qquad   \a^{(0)}=\x^{(spec)} \triangleq \mbox{max eigenvector} \lp A^T \mbox{diag} \lp  | A\bar{\x} |^2  \rp   A\rp.
\label{eq:practeq6}
\end{eqnarray}
The results obtained through numerical simulations  are shown in Figure \ref{fig:fig7}.  We chose $n=300$ and ran two procedures, $\mathbf{hybrid}$ and  $\mathbf{gradplain}$ with the same input $\a^{(0)}=\x^{(spec)}$. The obtained estimate of $\bar{\x}$ is denoted by $\hat{\x}$,
\begin{figure}[h]
\centering
\centerline{\includegraphics[width=1\linewidth]{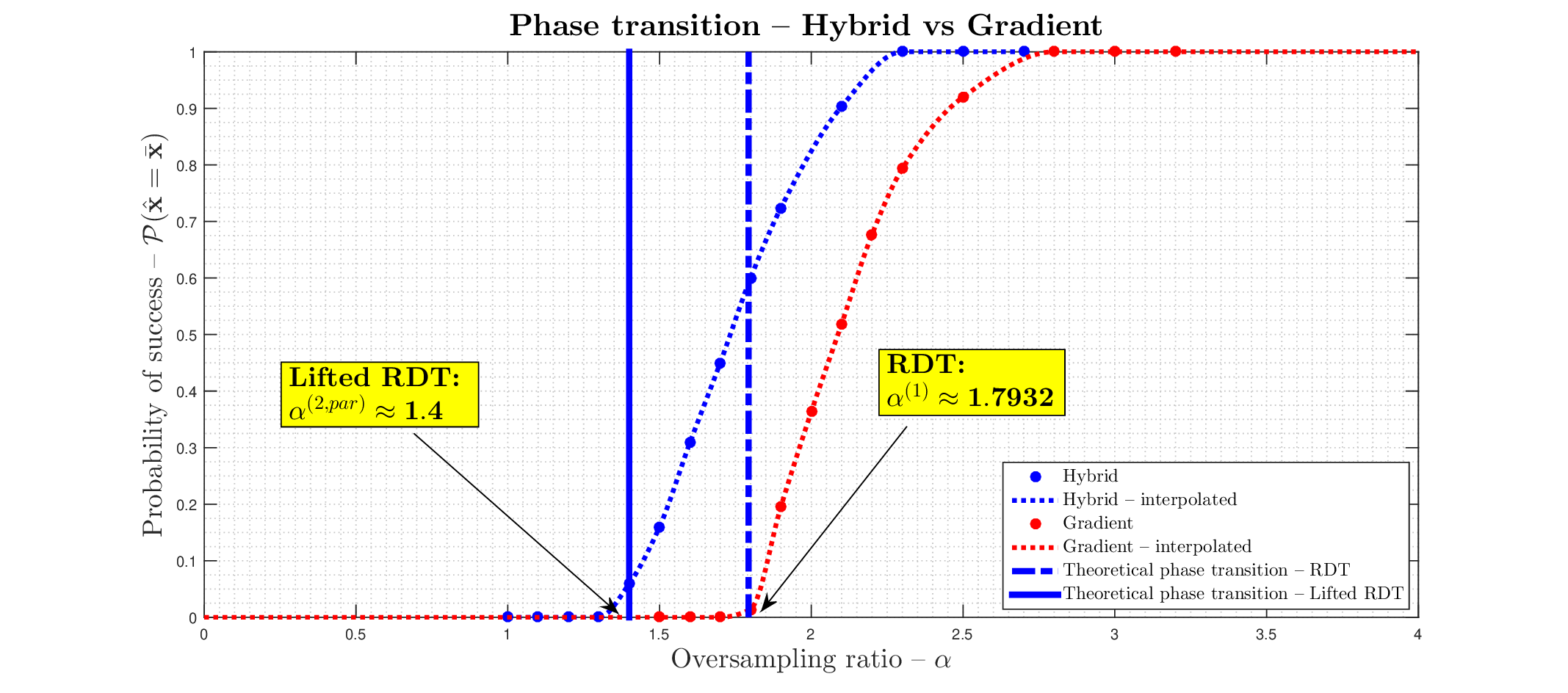}}
\caption{Simulated and theoretical RDT and lifted RDT phase transitions -- Hybrid vs Gradient}
\label{fig:fig7}
\end{figure}
As figure shows, simulated phase transition of the hybrid procedure is fairly close to the theoretical predictions (even though we simulated squared magnitudes objective opposed to the non-squared one used in theoretical calculations; however, we actually did not expect that such a change would make any substantial difference; see Sections \ref{sec:sqadj} and \ref{sec:sqliftrdt}). We should also add that the theoretical analysis is done for $f_{plain}$ with the norm constraint. Since we were running barrier variant the theoretical calculations need to be adjusted from $f_{plain}$ to $f_{bar}$. This is rather trivial. We have rechecked manifold shapes with $f_{bar}$ for a large set of $t_0$ and did not find that they structurally differ from the ones obtained for $f_{plain}$ (graphical representation of this for any $t_0$ would require 4D plots and as such is not feasible). Is Figures \ref{fig:fig1bar1} and  \ref{fig:fig1bar2} we show parametric manifolds of barrier functions $f_{bar}$ for $t_0=12$ and two different oversampling ratios $\alpha=1.4$ and $\alpha=1.7932$. We observe that $f_{bar}$ has very similar behavior as $f_{plain}$. In particular, for smaller $\alpha$ an undesired funneling point emerges. This is precisely along the lines of the above discussion. Further implementation of pl RDT (see Sections \ref{sec:sqadj} and \ref{sec:sqliftrdt}) actually shows that even for $\alpha=1.4$ the manifold gets sufficiently lifted that it has only a single funneling point -- again in a precise agreement with the above discussion.

\begin{figure}[h]
\centering
\centerline{\includegraphics[width=1\linewidth]{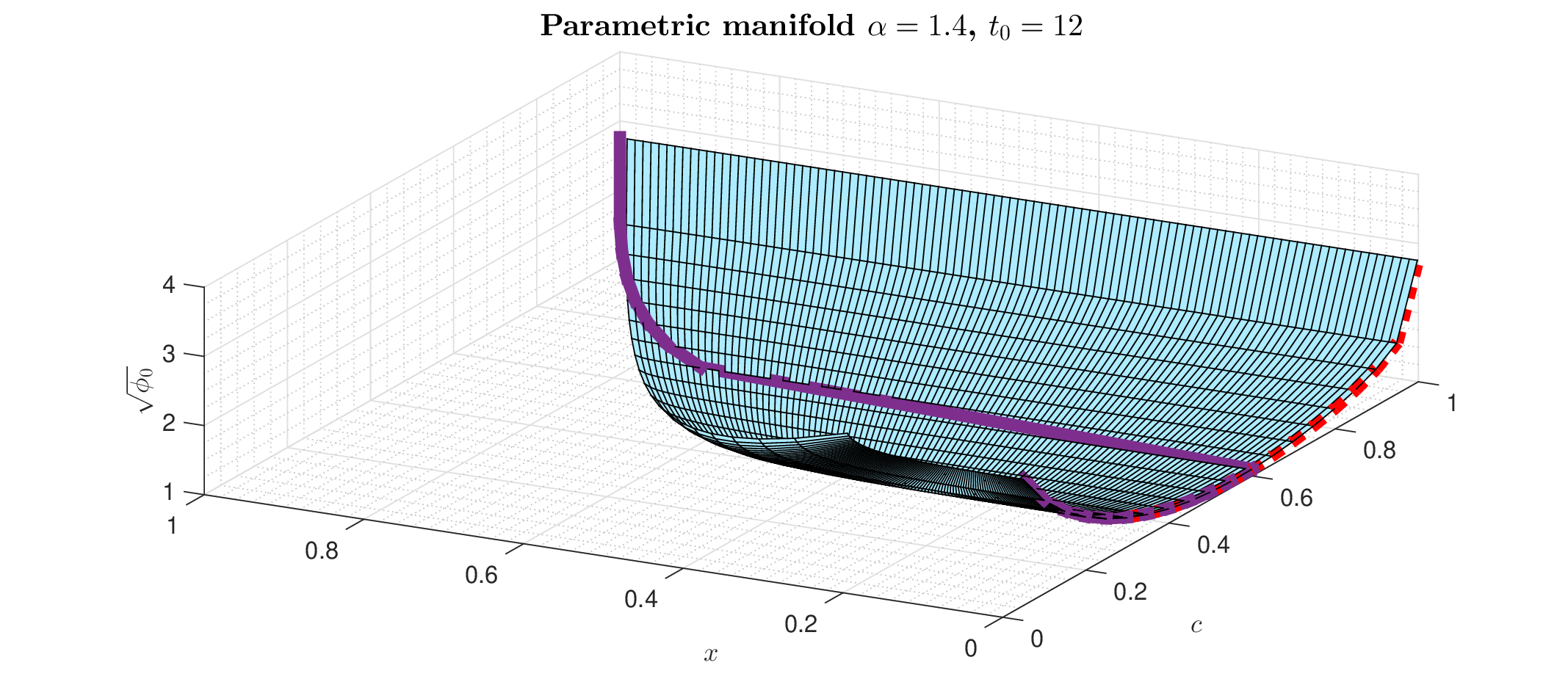}}
\caption{Barrier objective -- parametric manifold; $\alpha=1.4$, $t_0=12$; Purple curve -- path of local optima}
\label{fig:fig1bar1}
\end{figure}

\begin{figure}[h]
\centering
\centerline{\includegraphics[width=1\linewidth]{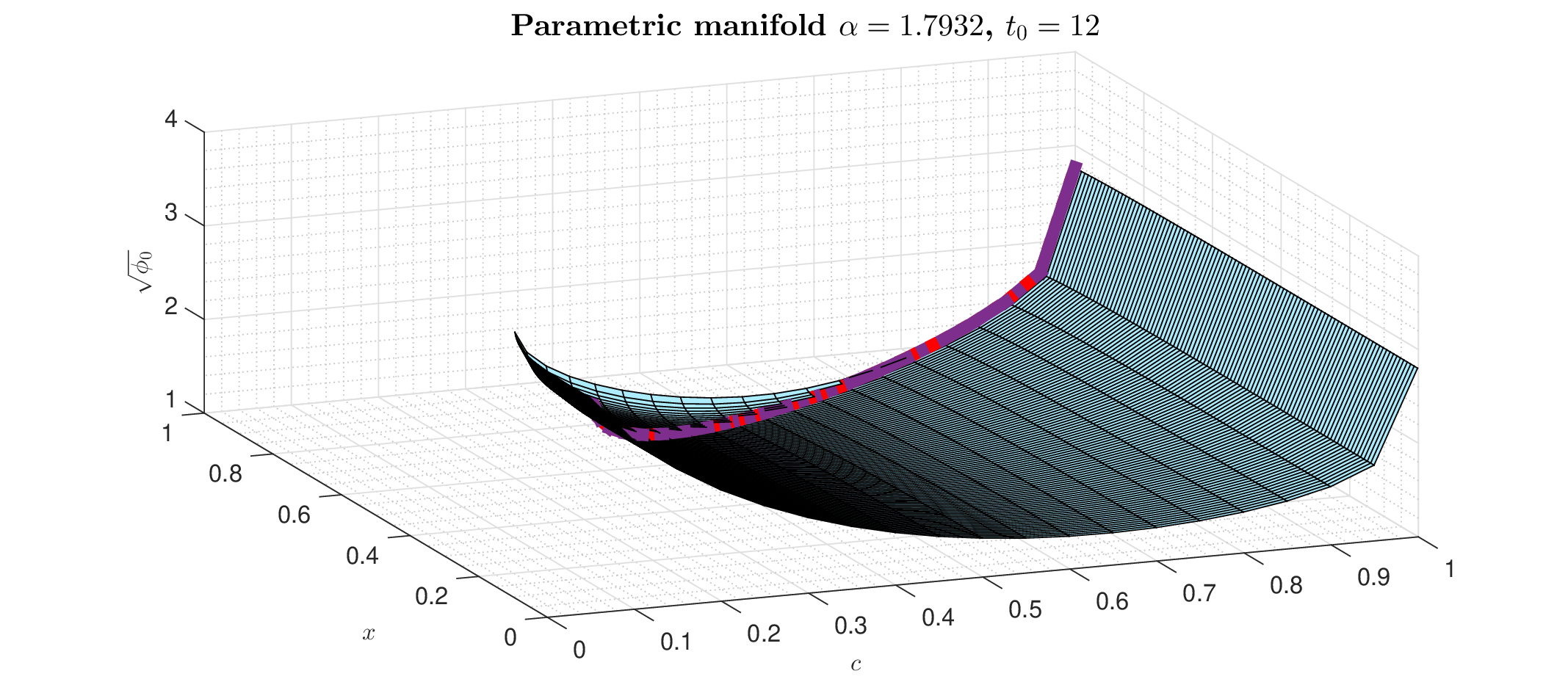}}
\caption{Barrier objective -- parametric manifold; $\alpha=1.7932$, $t_0=12$; Purple curve -- path of local optima}
\label{fig:fig1bar2}
\end{figure}

In our simulations it rarely took more than four iterations of the hybrid mechanism and its success typically coincided with the overlap, $x$, being  $\geq 0.8$ after the first iteration. It is also interesting to note that when we ran numerically a bit more ambitious setup with around 10 hybrid iterations we were able to solve instances with $\alpha\approx 1.5$ which is even closer to the theoretically predicted phase transition (this however takes enormous amount of time and we therefore opted to show in Figure \ref{fig:fig7} less favorable but easier to obtain results). All of this indicates  that parametric manifold is somewhat jittery and in a large portion of allowed $(c,x)$ values fairly flat when  $\alpha$ is close to the theoretical phase transition (flatness of the manifold makes the jittery effects even more pronounced).

Switching to the plain gradient, one observes that its simulated phase transitions are further away from the theoretical predictions. In addition to the above mentioned flat and jittery manifolds, a lack of control of the norm of the targeted vector, $\sqrt{c}$, is another key reason that moves the practical (simulated) phase transition away from the theoretical one. We discuss the effect of the norm constraint in more detail next.

\subsection{Why is the plain gradient away from theoretical phase transitions?}
 \label{sec:plaingrad}

Finite dimensions may imply occasional lack of concentrations which creates a manifold jitteriness and ultimately may cause descending algorithms to get trapped away from the desired solution. Running plain gradient, $\mathbf{gradplain}$, also allows $c=\|\x\|_2^2>1$ through the optimization process which may create additional theoretical and algorithmic obstacles. Namely, since the theoretical results presented in earlier sections did not allow for $c>1$ they consequently do not apply to the plain gradient. In Figure \ref{fig:fig8} we show how $\phi_0$ changes if $c$ is allowed to go beyond 1.
\begin{figure}[h]
\centering
\centerline{\includegraphics[width=1\linewidth]{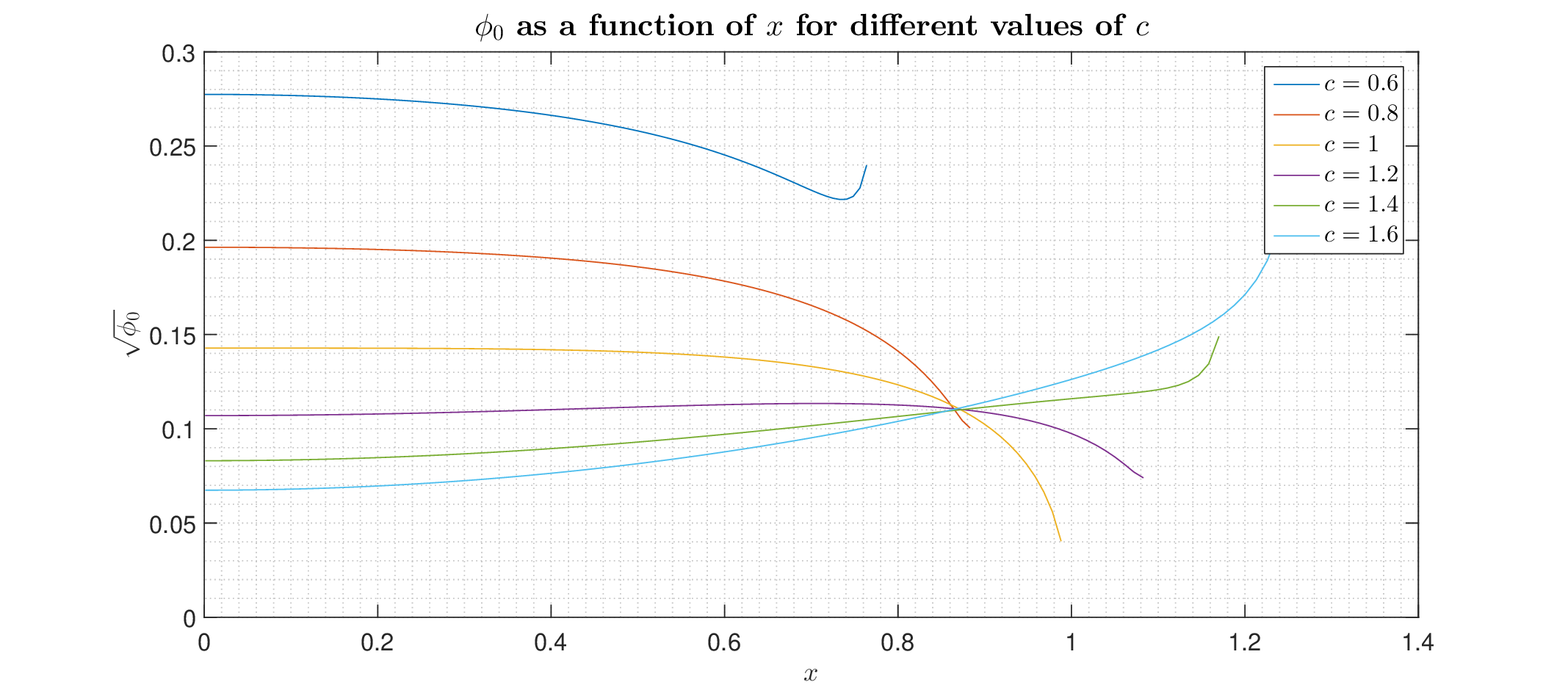}}
\caption{$\phi_0$ as a function of $x$; $c$ is allowed to go above 1; $\alpha=1.7932$}
\label{fig:fig8}
\end{figure}
As can be see, contour lines for $c>1$ start experiencing an unwanted boundary minimum at $x=0$  which (depending on the algorithm starting point) might be one of the (undesired) funneling points. Increasing oversampling ratio $\alpha$ is expected to remedy this. In Figure \ref{fig:fig9} we see that this is indeed the case. Figures \ref{fig:fig10} and \ref{fig:fig11} show the entire manifolds for $\alpha=1.7932$ and $\alpha=2.3$ and one can observe that for $\alpha=2.3$ the lower left corner (large $c$ and small $x$) the manifold is sufficiently lifted that only  one (desired) funneling point $(c,x)=(1,1)$ remains. Despite the fact that all the contour curves and manifold have a favorable shape for $\alpha=2.3$, one still can not make a generic conclusion that the plain gradient exhibits a phase transition phenomenon. In fact, we tested the shape of the contour curves and manifolds for much larger  $\alpha$'s and were always able to find $c>1$ such that the manifold has multiple funneling points.

\begin{figure}[h]
\centering
\centerline{\includegraphics[width=1\linewidth]{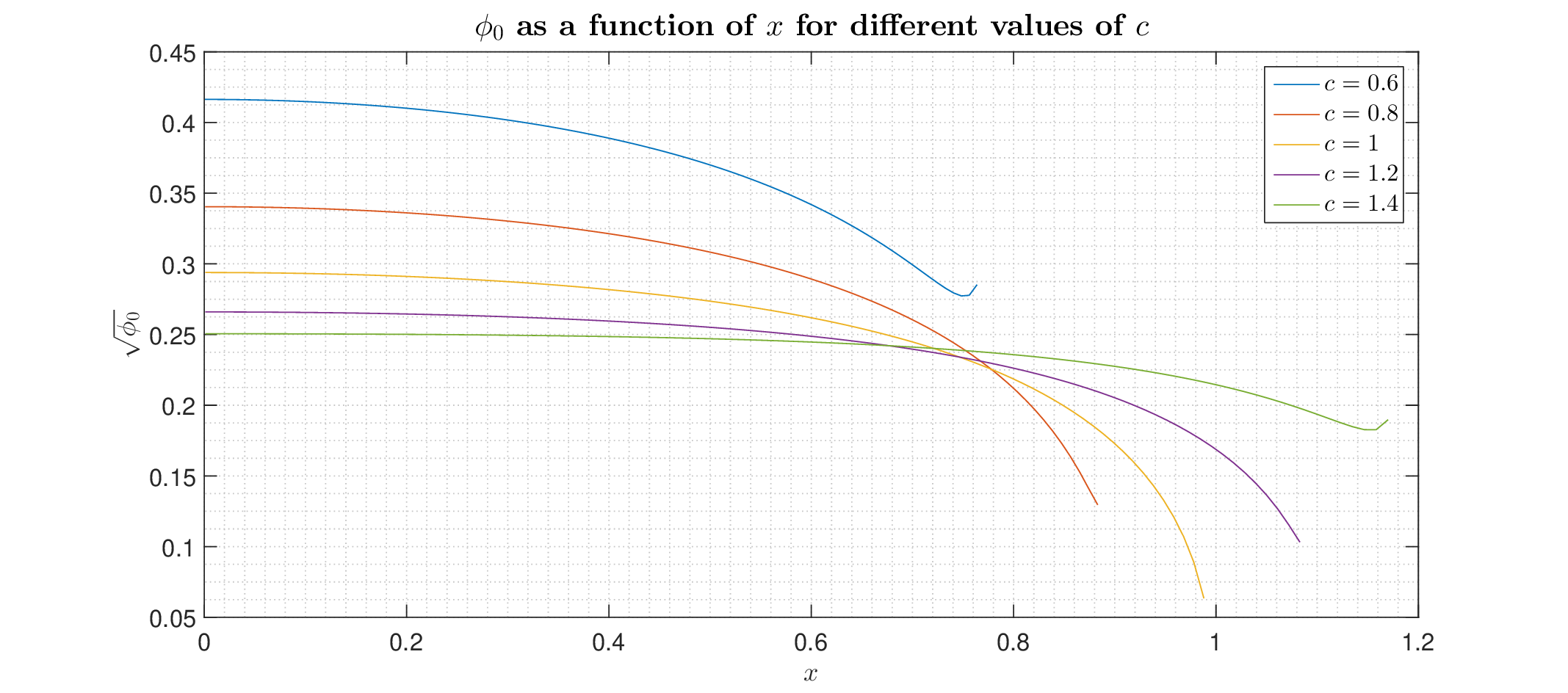}}
\caption{$\phi_0$ as a function of $x$; $c$ is allowed to go above 1; $\alpha=2.3$}
\label{fig:fig9}
\end{figure}

The above then raises a natural question: how does it happen that there are $\alpha$'s for which  everything actually works when one practically runs the plain gradient? The reason everything starts working for $\alpha\approx 2.3$ is that the algorithm itself rarely (if ever) reaches the region where $c>1.4$ (when we ran the $\mathbf{gradplain}$ for $\alpha=2.3$ we observed that maximum squared norm of $\x$ is indeed around $1.4$. As contour lines show, as long as $c<1.4$ there will be no multiple funneling points on the parametric manifold. This effectively means that while the $\mathbf{gradplain}$ does not have a generic phase transition its practical phase transition is in a very good agreement with the theoretical predictions.

One always has to keep in mind that the theoretical predictions are obtained through plain RDT which effectively bounds the true manifold from below. As shown earlier through the lifted RDT, the manifolds are actually expected do be lifted. However, we did not observe much of the lifting effect as $c$ moves away from $1$ (in earlier sections we presented results regarding this effect for $c<1$; we also observed similar trends for $c>1$). This basically means that not much of visible lifting correction (at least on the partial lifting level considered here) is expected for $c=1.4$ when $\alpha=2.3$. Of course, fl RDT might be able to lift things a bit more but given the numerical agreements with the simulated results we do not expect the effect such lifting might have on the location of practical phase transitions (those that account for algorithms' inabilities to reach $c>1.4$ region) to be more pronounced.

\begin{figure}[h]
\centering
\centerline{\includegraphics[width=1\linewidth]{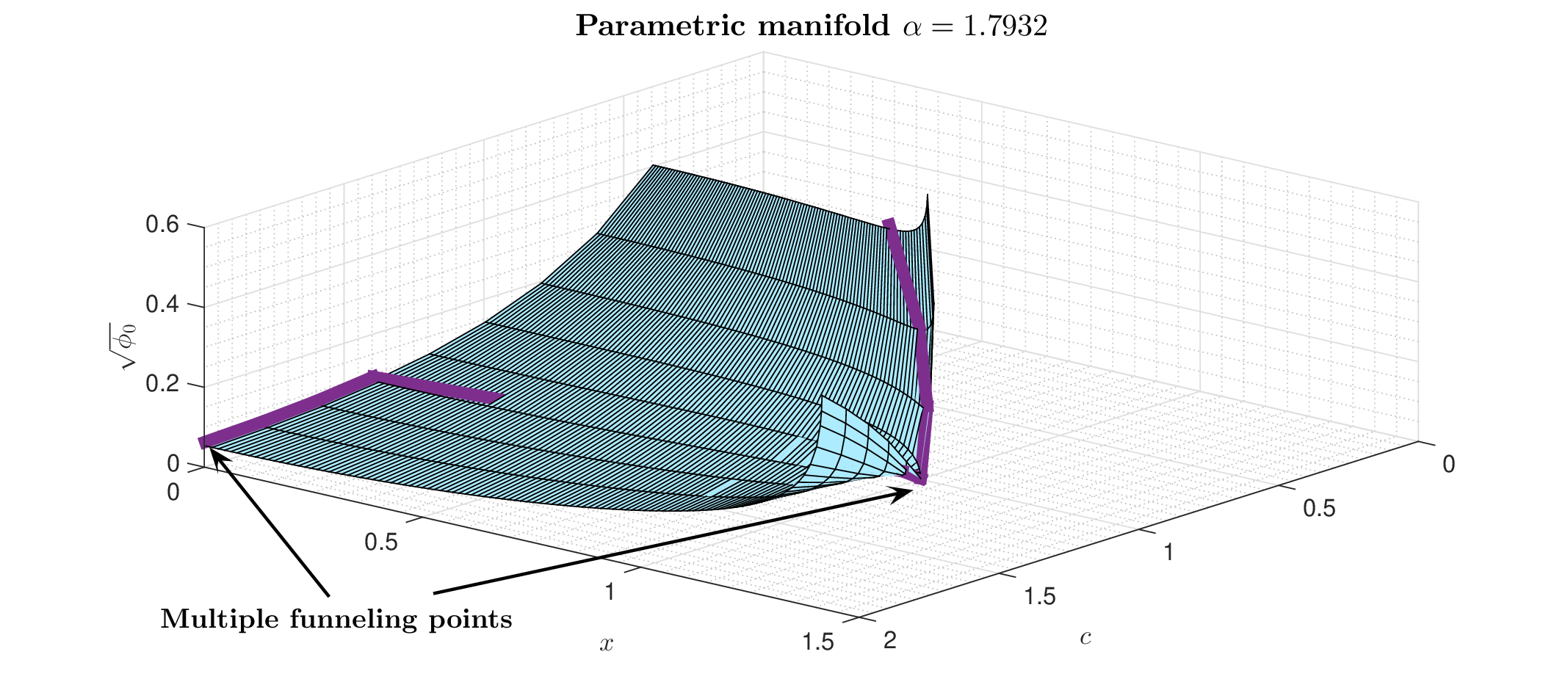}}
\caption{Parametric manifold for $\alpha=1.7932$; Purple curve -- path of local (or boundary) optima}
\label{fig:fig10}
\end{figure}

\begin{figure}[h]
\centering
\centerline{\includegraphics[width=1\linewidth]{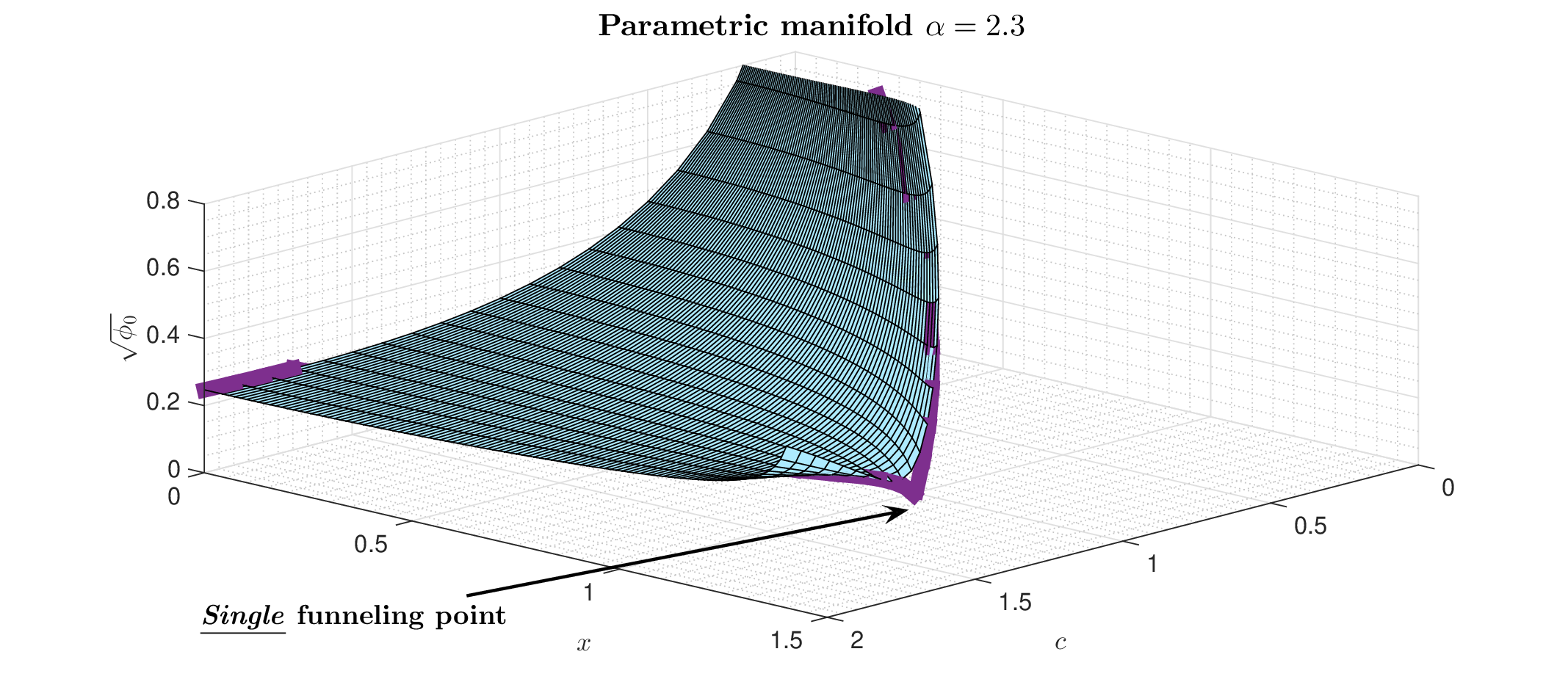}}
\caption{Parametric manifold for $\alpha=2.3$; Purple curve -- path of local (or boundary) optima}
\label{fig:fig11}
\end{figure}

\subsection{Squared magnitudes}
 \label{sec:sqadj}

In the above practical implementations we utilized
 \begin{eqnarray}
 f_{plain}(\x) = \||A\bar{\x}|^2 - |A\x|^2 \|_2^2.
 \label{eq:sqadj1}
\end{eqnarray}
which is slightly different from the one used in  (\ref{eq:ex1a4}) and (\ref{eq:ex1a4a0}) and the ensuing theoretical analyses from Sections \ref{sec:ubrdt} and \ref{sec:liftrdt}. Not much conceptually changes if instead of  (\ref{eq:ex1a4}) and (\ref{eq:ex1a4a0}) one in the theoretical analyses utilizes
 \begin{eqnarray}
 {\mathcal R}^{(sq)}A): \qquad \qquad    \min_{\x,\z} & & \||A\bar{\x}|^2-|\z|^2 \|_2^2\nonumber \\
  \mbox{subject to} & &  A\x=\z. \label{eq:sqadjex1a4}
\end{eqnarray}
and
 \begin{eqnarray}
\hspace{-.8in}\bl{\textbf{\emph{f-pro -- squared magnitudes:}}} \qquad\qquad  \xi^{(sq)}(c,x) \triangleq \min_{\x,\z} & & \||A\bar{\x}|^2-|\z|^2 \|_2^2\nonumber \\
  \mbox{subject to} & &  A\x=\z \nonumber \\
  & & \x^T\bar{\x}=x \nonumber \\
  & & \|\x\|_2^2=c. \label{eq:sqadjex1a4a0}
\end{eqnarray}
Clearly, the only difference between (\ref{eq:ex1a4}) and (\ref{eq:ex1a4a0}) on the one side and (\ref{eq:sqadjex1a4}) and (\ref{eq:sqadjex1a4a0}) on the other is in \emph{squared  magnitudes} in the objectives. While this brings no conceptual complications regarding theoretical analysis it has potentially massive consequences on the accuracy of the residual numerical evaluations which should be taken with a solid degree of caution.

To facilitate the presentation that follows, we mimic what was done in Section \ref{sec:ubrdt} but proceed in a much faster fashion avoiding unnecessary repetitions of the already discussed ideas. We start with the following \emph{squared magnitudes} analogue to Theorem \ref{thm:thm1} and  whenever possible throughout ensuing derivations adopt the principle of trying to recreate quantities analogous to those discussed in the non-squared magnitudes scenarios.

\begin{theorem} Assume the setup of Lemma \ref{lemma:lemma1}. Let the elements of $A\in\mR^{m\times n}$ ($\g^{(0)}\in\mR^{m\times 1}$ and $A_{:,2:n}\in\mR^{m\times (n-1)}$), $\g^{(1)}\in\mR^{m\times 1}$, and  $\h^{(1)}\in\mR^{(n-1)\times 1}$  be iid standard normals. Consider two positive scalars $c$ and $x$  ($0\leq x \leq c$) and set $r\triangleq \sqrt{c-x^2}$. Let
\vspace{-.0in}
\begin{eqnarray}
f^{(sq)}_{rp} & = & \frac{\xi^{(sq)}(c,x)}{n} \nonumber\\
\cG & \triangleq & \lp A,\g^{(1)},\h^{(1)}\rp = \lp\g^{(0)},A_{:,2:n},\g^{(1)},\h^{(1)}\rp  \nonumber \\
\phi^{(sq)}(\x,\z,\y) & \triangleq &
 \lp \||\g^{(0)}|^2-|\z|^2 \|_2^2
  +\y^T \g^{(0)}x   +  \y^T \g^{(1)}\|\x_{2:n}\|_2 + \lp \x_{2:n}  \rp^T\h^{(1)}\|\y\|_2  -\y^T  \z  \rp
\nonumber \\
 f_{rd}^{(sq)}(c,x;\cG) & \triangleq &
\frac{1}{n}  \min_{\|\x_{2:n}\|_2=r,\z} \max_{\|\y\|_2=r_y,r_y>0}  \phi^{(sq)}(\x,\z,\y)
  \nonumber \\
 \phi_0^{(sq)} & \triangleq & \lim_{n\rightarrow\infty} \mE_{\cG} f^{(sq)}_{rd}(c,x;\cG).\label{eq:sqadjtsqadja16}
\vspace{-.0in}\end{eqnarray}
One then has \vspace{-.02in}
\begin{eqnarray}
  \lim_{n\rightarrow\infty}\mP_{ A } (f^{(sq)}_{rp} (c,x; A )   >  \phi^{(sq)}_0)\longrightarrow 1,\label{eq:sqadjta17a0}
\end{eqnarray}
and
\begin{eqnarray}
\hspace{-.3in}(\phi^{(sq)}_0  > 0)   &  \Longrightarrow  & \lp \lim_{n\rightarrow\infty}\mP_{\cG}\lp \frac{\xi^{(sq)}(c,x)}{n} = f^{(sq)}_{rd}(c,x;\cG) >0 \rp \longrightarrow 1\rp \nonumber \\
&\Longrightarrow & \quad \lp \lim_{n\rightarrow\infty}\mP_{ A } (f^{(sq)}_{rp} (c,x; A )   >0)\longrightarrow 1 \rp  \nonumber \\
& \Longrightarrow & \lp \lim_{n\rightarrow\infty}\mP_{A} \lp \mbox{PR is (uniquely) solvable} \rp \longrightarrow 1\rp.\label{eq:sqadjta17}
\end{eqnarray}
 \label{thm:thm3}
\end{theorem}\vspace{-.17in}
\begin{proof}
Immediate consequence of Theorem \ref{thm:thm1} after adjusting for the squared magnitudes of $\g^{(0)}$ and $\z$ in the definition of $\phi^{(sq)}(\x,\z,\y)$.
\end{proof}

We can then proceed to handle the random dual as in the third part of Section \ref{sec:ubrdt}. Analogously to  (\ref{eq:hrd6} and  (\ref{eq:hrd7}, we first write
\begin{eqnarray}
 \phi^{(sq)}_0 & \triangleq & \lim_{n\rightarrow\infty} \mE_{\cG} f^{(sq)}_{rd}(\cG)
\geq
  \max_{r_y>0} \mE_{\cG}  \min_{\z_i} \cL^{(sq)}_1(r_y),
\label{eq:sqadjhrd6}
 \end{eqnarray}
where
\begin{eqnarray}
\cL^{(sq)}_1(r_y)
  &  = &
\alpha  \lp \||\g_i^{(0)}|^2-|\z_i|^2 \|_2^2
  +\||\g_i^{(0)}x   +  \g_i^{(1)}r| -|\z_i|\|_2^2 r_y \rp
  -  r^2 r_y.
\label{eq:sqadjhrd7}
 \end{eqnarray}
To optimize over $\z_i$ we first find the following derivative
\begin{eqnarray}
\frac{d\cL^{(sq)}_1(r_y)}{d|\z_i|}
  &  = &
\alpha \lp -4|\z_i|(|\g_i^{(0)}|^2-|\z_i|^2 )
  -2(|\g_i^{(0)}x   +  \g_i^{(1)}r| -|\z_i|) r_y\rp.
\label{eq:hrd7a0}
 \end{eqnarray}
 Equalling the above derivative to zero gives the following cubic equation
\begin{eqnarray}
 |\z_i|^3 + p_c |\z_i| + q_c = 0,
\label{eq:hrd7a0a0}
 \end{eqnarray}
 where
\begin{eqnarray}
p_c & = & \frac{1}{2}\lp r_y -2|\g^{(0)}| \rp \nonumber \\
q_c & = &  -\frac{r_y}{2} |\g_i^{(0)}x   +  \g_i^{(1)}r |.
\label{eq:hrd7a0a1}
 \end{eqnarray}
After setting
\begin{eqnarray}
     a_{c,1} &  =  & \lp -\frac{q_c}{2} + \sqrt{ \frac{q_c^2}{4} + \frac{p_c^3}{27} }  \rp^{\frac{1}{3}}
 - \frac{p_c}{3     \lp -\frac{q_c}{2} + \sqrt{ \frac{q_c^2}{4} + \frac{p_c^3}{27} }  \rp^{\frac{1}{3}}  }
\nonumber \\
     a_{c,2} &  =  &  a_{c,1}  \frac{-1-\sqrt{-3}}{2}
     \nonumber \\
      a_{c,3} &  =  &  a_{c,1}  \frac{-1+\sqrt{-3}}{2},
\label{eq:hrd7a0a2}
 \end{eqnarray}
one obtains through the Cardano's formula sets of possible candidates for  optimal $|\z_i|$
\begin{eqnarray}
{\mathcal Z}^{(sq)} =
\begin{cases}
  \{a_{c,1},0 \}, & \mbox{if $\frac{q_c^2}{4} + \frac{p_c^3}{27} \geq 0$}  \\
  \{a_{c,1},a_{c,2},a_{c,3},0 \} , & \mbox{otherwise}.
\end{cases}.
\label{eq:sqadjhrd7a1}
 \end{eqnarray}
Combining  (\ref{eq:sqadjhrd7}) and  (\ref{eq:sqadjhrd7a1}) one further obtains
\begin{eqnarray}
\min_{\z_i} \cL^{(sq)}_1(r_y)
  &  = &
\alpha  \min_{|\z_i|\in{\mathcal Z}^{(sq)}}\lp \||\g_i^{(0)}|^2-|\z_i|^2 \|_2^2
  +\||\g_i^{(0)}x   +  \g_i^{(1)}r| -|\z_i|\|_2^2 r_y \rp
  -  r^2 r_y.
\label{eq:sqadjhrd7a2}
 \end{eqnarray}
 Setting
\begin{eqnarray}
f_q^{(sq)} & \triangleq & \mE_{\cG} \min_{|\z_i|\in{\mathcal Z}^{(sq)}}\lp \||\g_i^{(0)}|^2-|\z_i|^2 \|_2^2
  +\||\g_i^{(0)}x   +  \g_i^{(1)}r| -|\z_i|\|_2^2 r_y \rp,
\label{eq:hrd7a2a0}
 \end{eqnarray}
 easily gives
 \begin{eqnarray}
 \mE_{\cG} \min_{\z_i} \cL^{(sq)} _1(r_y)
     &  = &
\alpha
\mE_{\cG} \min_{|\z_i|\in{\mathcal Z}^{(sq)}}\lp \||\g_i^{(0)}|^2-|\z_i|^2 \|_2^2
  +\||\g_i^{(0)}x   +  \g_i^{(1)}r| -|\z_i|\|_2^2 r_y \rp
  -  r^2 r_y
\nonumber \\
      &  = &
\alpha
  f^{(sq)}_q  -  r^2 r_y.
\label{eq:hrd7a3}
 \end{eqnarray}
Plugging this back in (\ref{eq:sqadjhrd6}) gives
\begin{eqnarray}
 \phi^{(sq)}_0 & \triangleq & \lim_{n\rightarrow\infty} \mE_{\cG} f_{rd}(\cG)
 \nonumber \\
& \geq &
  \max_{r_y>0}
  \lp
\alpha
\mE_{\cG} \min_{|\z_i|\in{\mathcal Z}^{(sq)}}\lp \||\g_i^{(0)}|^2-|\z_i|^2 \|_2^2
  +\||\g_i^{(0)}x   +  \g_i^{(1)}r| -|\z_i|\|_2^2 r_y \rp
  -  r^2 r_y
   \rp
  \nonumber \\
  & = &
  \max_{r_y>0}
  \lp
  \alpha
  f^{(sq)}_q
  -  r^2 r_y \rp.
\label{eq:sqadjhrd7a4}
 \end{eqnarray}

\subsection{Squared magnitudes -- Lifted RDT}
 \label{sec:sqliftrdt}

The above RDT analysis can be lifted relying on the concepts from Section \ref{sec:liftrdt}. One starts by establishing the following \emph{squared magnitudes} analogue to Theorem \ref{thm:thm2}).

\begin{theorem} Assume the setup of Theorem \ref{thm:thm1} with the elements of $A\in\mR^{m\times n}$ ($\g^{(0)}\in\mR^{m\times 1}$ and $A_{:,2:n}\in\mR^{m\times (n-1)}$), $\g^{(1)}\in\mR^{m\times 1}$, and  $\h^{(1)}\in\mR^{(n-1)\times 1}$  being iid standard normals. Consider two positive scalars $c$ and $x$  ($0\leq x \leq c$) and set $r\triangleq \sqrt{c-x^2}$. Let $c_3>0$ and
\vspace{-.0in}
\begin{eqnarray}
\cG & \triangleq & \lp A,\g^{(1)},\h^{(1)}\rp = \lp\g^{(0)},A_{:,2:n},\g^{(1)},\h^{(1)}\rp  \nonumber \\
\phi^{(sq)}(\x,\z,\y) & \triangleq &
 \lp \||\g^{(0)}|^2-|\z|^2 \|_2^2
  +\y^T \g^{(0)}x   +  \y^T \g^{(1)}\|\x_{2:n}\|_2 + \lp \x_{2:n}  \rp^T\h^{(1)}\|\y\|_2  -\y^T  \z  \rp
\nonumber \\
 \bar{f}^{(sq)}_{rd}(c,x;\cG) & \triangleq &
 \min_{\|\x_{2:n}\|_2=r,\z} \max_{\|\y\|_2=r_y}  \phi^{(sq)}(\x,\z,\y)
  \nonumber \\
  \bar{\phi}^{(sq)}_0 & \triangleq & \max_{r_y>0}\lim_{n\rightarrow\infty} \frac{1}{n}
 \lp
 \frac{c_3}{2} r^2r_y^2 -
\frac{1}{c_3} \log \lp \mE_{\cG_{(2)}} e^{ - c_3 \bar{f}^{(sq)}_{rd}(\cG) } \rp   \rp .\label{eq:sqadjplta16}
\vspace{-.0in}\end{eqnarray}
One then has \vspace{-.02in}
\begin{eqnarray}
  \lim_{n\rightarrow\infty}\mP_{ A } \lp \frac{\xi^{(sq)} (c,x) }{n} =  f^{(sq)}_{rp} (c,x; A )   > \bar{\phi}^{(sq)}_0 \rp \longrightarrow 1,\label{eq:sqadjplta17a0}
\end{eqnarray}
and
\begin{eqnarray}
\hspace{-.0in}(\phi^{(sq)}_0  > 0)    \Longrightarrow \lp \lim_{n\rightarrow\infty}\mP_{ A } (f^{(sq)}_{rp} (c,x; A )   >0)\longrightarrow 1 \rp
 \Longrightarrow \lp \lim_{n\rightarrow\infty}\mP_{A} \lp \mbox{PR is (uniquely) solvable} \rp \longrightarrow 1\rp.\label{eq:sqadjplta17}
\end{eqnarray}
 \label{thm:thm4}
\end{theorem}\vspace{-.17in}
\begin{proof}
It is an immediate consequence of Theorem \ref{thm:thm2} in exactly the same way Theorem \ref{thm:thm3} is an immediate consequence of Theorem \ref{thm:thm1}.
\end{proof}

One can then proceed as in Section \ref{sec:liftrdt} to handle the above partially lifted random dual. We first write analogously to  (\ref{eq:plhrd2})
 \begin{eqnarray}
\bar{f}^{(sq)}_{rd}(c,x;\cG)
 & = &
\min_{\z,\gamma>0}\max_{\gamma_{sph}>0}
 \bar{\cL}(r_y)
 \label{eq:sqadjplhrd2}
 \end{eqnarray}
where
\begin{eqnarray}
\bar{\cL}^{(sq)}(r_y) = \lp  \sum_{i=1}^{m}\||\g_i^{(0)}|^2-|\z_i|^2 \|_2^2
+\gamma   + \frac{\sum_{i=1}^{m}\|\g^{(0)}x   +  \g^{(1)}r -\z\|_2^2 r_y^2}{4\gamma} - \gamma_{sph} - \frac{\sum_{i=1}^{n}\|\h^{(1)} \|_2^2 r^2 r_y^2 }{4\gamma_{sph}}   \rp.
 \label{eq:sqadjplhrd3}
 \end{eqnarray}
After setting
\begin{eqnarray}
\bar{r}_y & =  & \frac{r_y^2}{4\gamma} \nonumber \\
\bar{p}_c & = & \frac{1}{2}\lp \bar{r}_y -2|\g^{(0)}| \rp \nonumber \\
\bar{q}_c & = &  -\frac{\bar{r}_y}{2} |\g_i^{(0)}x   +  \g_i^{(1)}r |.
\nonumber \\
     \bar{a}_{c,1} &  =  & \lp -\frac{\bar{q}_c}{2} + \sqrt{ \frac{\bar{q}_c^2}{4} + \frac{\bar{p}_c^3}{27} }  \rp^{\frac{1}{3}}
 - \frac{\bar{p}_c}{3     \lp -\frac{\bar{q}_c}{2} + \sqrt{ \frac{\bar{q}_c^2}{4} + \frac{\bar{p}_c^3}{27} }  \rp^{\frac{1}{3}}  }
\nonumber \\
     \bar{a}_{c,2} &  =  &  \bar{a}_{c,1}  \frac{-1-\sqrt{-3}}{2}
     \nonumber \\
      \bar{a}_{c,3} &  =  &  \bar{a}_{c,1}  \frac{-1+\sqrt{-3}}{2}
      \nonumber \\
\bar{{\mathcal Z}}_i^{(sq)} &= &
\begin{cases}
  \{\bar{a}_{c,1},0 \}, & \mbox{if $\frac{\bar{q}_c^2}{4} + \frac{\bar{p}_c^3}{27} \geq 0$}  \\
  \{\bar{a}_{c,1},\bar{a}_{c,2},\bar{a}_{c,3},0 \} , & \mbox{otherwise}.
\end{cases},
 \label{eqsqadj:plhrd4}
\end{eqnarray}
 one has
  \begin{equation}
\min_{\z} \bar{\cL}^{(sq)}(r_y)
 = \hspace{-.12in}
\min_{\z_i\in\bar{{\mathcal Z}}_i^{(sq)}}  \lp  \sum_{i=1}^{m}
\lp \|\g_i^{(0)}|^2-|\z_i|^2 \|_2^2 + \|\g^{(0)}x   +  \g^{(1)}r -\z\|_2^2 \bar{r}_y\rp
+\gamma    - \gamma_{sph} \hspace{-.03in} - \frac{\sum_{i=1}^{n}\|\h^{(1)} \|_2^2 r^2 r_y^2 }{4\gamma_{sph}}   \rp.
\label{eq:sqadjplhrd6}
 \end{equation}
Combining  (\ref{eq:sqadjplhrd2}) and (\ref{eq:sqadjplhrd6}) we furthr obtain
\begin{eqnarray}
\bar{f}^{(sq)}_{rd}(c,x;\cG)
 & = &
\min_{\z,\gamma>0}\max_{\gamma_{sph}>0}
 \bar{\cL}(r_y)
  \nonumber \\
 & = &
\min_{\gamma>0}\max_{\gamma_{sph}>0}
  \Bigg(\Bigg.
   \min_{\z_i\in\bar{{\mathcal Z}}_i^{(sq)}}  \lp  \sum_{i=1}^{m}
\lp \|\g_i^{(0)}|^2-|\z_i|^2 \|_2^2 + \|\g^{(0)}x   +  \g^{(1)}r -\z\|_2^2 \bar{r}_y\rp
\rp
\nonumber \\
& & +\gamma   - \gamma_{sph} - \frac{\sum_{i=1}^{n}\|\h^{(1)} \|_2^2 r^2 r_y^2 }{4\gamma_{sph}}   \Bigg. \Bigg).
  \label{eq:sqadjplhrd7}
 \end{eqnarray}
 After applying  appropriate scaling   $\gamma \rightarrow \gamma n$, $\gamma_{sph}\rightarrow \gamma_{sph}n $, $r_y\rightarrow r_y\sqrt{n}$, concentrations, statistical identicalness over $i$, Lagrangian duality, and a combination of (\ref{eq:sqadjplta16}) and (\ref{eq:sqadjplhrd7}) give
\begin{eqnarray}
  \bar{\phi}^{(sq)}_0 & \triangleq & \max_{r_y>0}\lim_{n\rightarrow\infty} \frac{1}{n}
 \lp
 \frac{c_3}{2} r^2r_y^2 -
\frac{1}{c_3} \log \lp \mE_{\cG_{(2)}} e^{ - c_3 \bar{f}^{(sq)}_{rd}(\cG) } \rp   \rp
 \nonumber \\
& \geq &
\max_{r_y>0}\min_{\gamma>0} \max_{\gamma_{sph}>0}
\lp
 \frac{c_3}{2} r^2r_y^2 + \gamma
 -\frac{\alpha}{c_3} \log \lp \mE_{\cG} e^{ - c_3 \bar{f}^{(sq)}_{q}  } \rp
 - \gamma_{sph}  - \frac{1}{c_3} \log \lp \mE_{\cG_{(2)}} e^{ c_3 \frac{\lp \h_i^{(1)}\rp^2r^2r_y^2}{4\gamma_{sph}} }\rp
\rp, \nonumber \\
\label{eq:sqadjplhrd8}
 \end{eqnarray}
where
 \begin{eqnarray}
\bar{f}_q^{(sq)} & \triangleq & \mE_{\cG} \min_{|\z_i|\in{\bar{\mathcal Z}}_i^{(sq)}}\lp \||\g_i^{(0)}|^2-|\z_i|^2 \|_2^2
  +\||\g_i^{(0)}x   +  \g_i^{(1)}r| -|\z_i|\|_2^2 \bar{r}_y \rp.
 \label{eq:sqadjplhrd9}
 \end{eqnarray}
After setting
\begin{eqnarray}
f_{q}^{(sq,lift)}  & = &  \mE_{\cG} e^{-c_3 \bar{f}^{(sq)}_{q}}
\nonumber \\
 & = &
\int_{-\infty}^{\infty}
\int_{-\infty}^{\infty}
 \frac{e^{
 -c_3 \lp \min_{|\z_i|\in{\bar{\mathcal Z}}_i^{(sq)}}\lp \||\g_i^{(0)}|^2-|\z_i|^2 \|_2^2
  +\||\g_i^{(0)}x   +  \g_i^{(1)}r| -|\z_i|\|_2^2 \bar{r}_y \rp \rp
 -\frac{ \lp \g_i^{(0)} \rp^2 + \lp \g_i^{(1)} \rp^2     } {2}  }}{\sqrt{2\pi}^2} d\g_i^{(0)} d\g_i^{(1)}, \nonumber \\
  \label{eq:sqadjplhrd11}
 \end{eqnarray}
and recalling on (\ref{eq:plhrd12}), one obtains through a connection of (\ref{eq:sqadjplhrd8})-(\ref{eq:sqadjplhrd11}) the following \emph{squared magnitudes} analogue to (\ref{eq:plhrd14})
\begin{eqnarray}
   \bar{\phi}^{(sq)}_0
& \geq &
\max_{r_y>0}\min_{\gamma>0} \max_{\gamma_{sph}>0}
\lp
 \frac{c_3}{2} r^2r_y^2 + \gamma
 -\frac{\alpha}{c_3} \log \lp \mE_{\cG} e^{ - c_3 \bar{f}^{(sq)}_{q}  } \rp
 - \gamma_{sph}  - \frac{1}{c_3} \log \lp \mE_{\cG_{(2)}} e^{ c_3 \frac{\lp \h_i^{(1)}\rp^2r^2r_y^2}{4\gamma_{sph}} }\rp
\rp
 \nonumber \\
 & = &
\max_{r_y>0}\min_{\gamma>0}
\Bigg .\Bigg(
 \frac{c_3}{2} r^2r_y^2 + \gamma
 -\frac{\alpha}{c_3} \log \lp f_{q}^{(sq,lift)}\rp
 - \hat{\gamma}_{sph}  +\frac{1}{2c_3} \log \lp  1  -  \frac{c_{3,e}}{2\hat{\gamma}_{sph}}     \rp
\Bigg.\Bigg),
\label{eq:sqadjplhrd14}
 \end{eqnarray}
with  $\hat{\gamma}_{sph}$  as in (\ref{eq:plhrd13}). Keeping in mind that (\ref{eq:sqadjplhrd14}) holds for any $c_3$, additional maximization of the right hand side over $c_3$ further gives
\begin{eqnarray}
 \bar{\phi}^{(sq)}_0
& \geq &
 \max_{c_3> 0}  \max_{r_y>0}\min_{\gamma>0}
\Bigg .\Bigg(
 \frac{c_3}{2} r^2r_y^2 + \gamma
 -\frac{\alpha}{c_3} \log \lp f_{q}^{(sq,lift)}\rp
 - \hat{\gamma}_{sph}  +\frac{1}{2c_3} \log \lp  1  -  \frac{c_3rr_y  }{2\hat{\gamma}_{sph}}     \rp
\Bigg.\Bigg).
\label{eq:sqadjplhrd16a0}
 \end{eqnarray}

Numerical results regarding the above analyses are shown in Figure \ref{fig:fig12}. We conducted numerical evaluations for  $\alpha=1.4$. As we have seen in the previous sections,   the choice $c=1$ is critical for determining the phase transitions in the  non-squared magnitudes case. The same trend continues here and we selected $c=1$ for the numerical evaluations to visualize the effect of the above results. As figure  shows, the effect of lifted RDT is rather strong and it manages to flatten the curve so that there is no undesired local minimum in the lower left corner. As discussed earlier, this then implies that any norm constrained descending phase retrieval algorithm applied on the squared magnitudes $f_{plain} (\x)$ will not get trapped in a local optimum.

In Sections \ref{sec:ubrdt} and \ref{sec:liftrdt} we considered the non-squared magnitudes scenario. In such a context several underlying optimizations could be solved in closed form which made numerical evaluations substantially easier and likely more accurate. Here on the other hand, one does not have such a luxury
and a sizeable set of additional numerical integrations is needed which  may incur additional residual errors. Given that the lifted RDT curve plotted in Figure \ref{fig:fig12} is fairly flat (with the obtained difference in $\sqrt{\bar{\phi}^{(sq)}_0}$ often on the fifth decimal for a large range of $c$) it is a bit difficult to make a definite conclusion that $\alpha\approx 1.4$ is indeed sufficient for a nondecreasing behavior of $\sqrt{\bar{\phi}^{(sq)}_0}$. To be on a safe side in practice one should use slightly larger $\alpha$.  Nonetheless, a rather substantial lifting effect is observed. Moreover, it indicates that there is not that much of a conceptual algorithmic difference between non-squared and squared magnitudes in the objective $f_{plain}(\x)$.. Two key takeaway differences are: \textbf{\emph{(i)}}  Due to smoothness and continuity of the derivatives the squared option is easier to use in practical
implementations; and \textbf{\emph{(ii)}} As it admits closed form solutions to  some of the residual optimizations, the non-squared option is more useful in theoretical considerations. On the other hand, a common takeaway is the existence of fairly large flat portions of the contour lines.  Coupled with finite dimensions utilized in practical algorithmic running and consequential local jitteriness of the parametric manifolds this may easily tilt manifold's flat portions in an undesired direction ultimately causing emergence of unwanted funneling points and convergence of the descending algorithms to a non-global optimum. As stated above, it is then reasonable to use slightly larger values of $\alpha$ in practical running -- precisely as done in Figure \ref{fig:fig7}.

\begin{figure}[h]
\centering
\centerline{\includegraphics[width=1\linewidth]{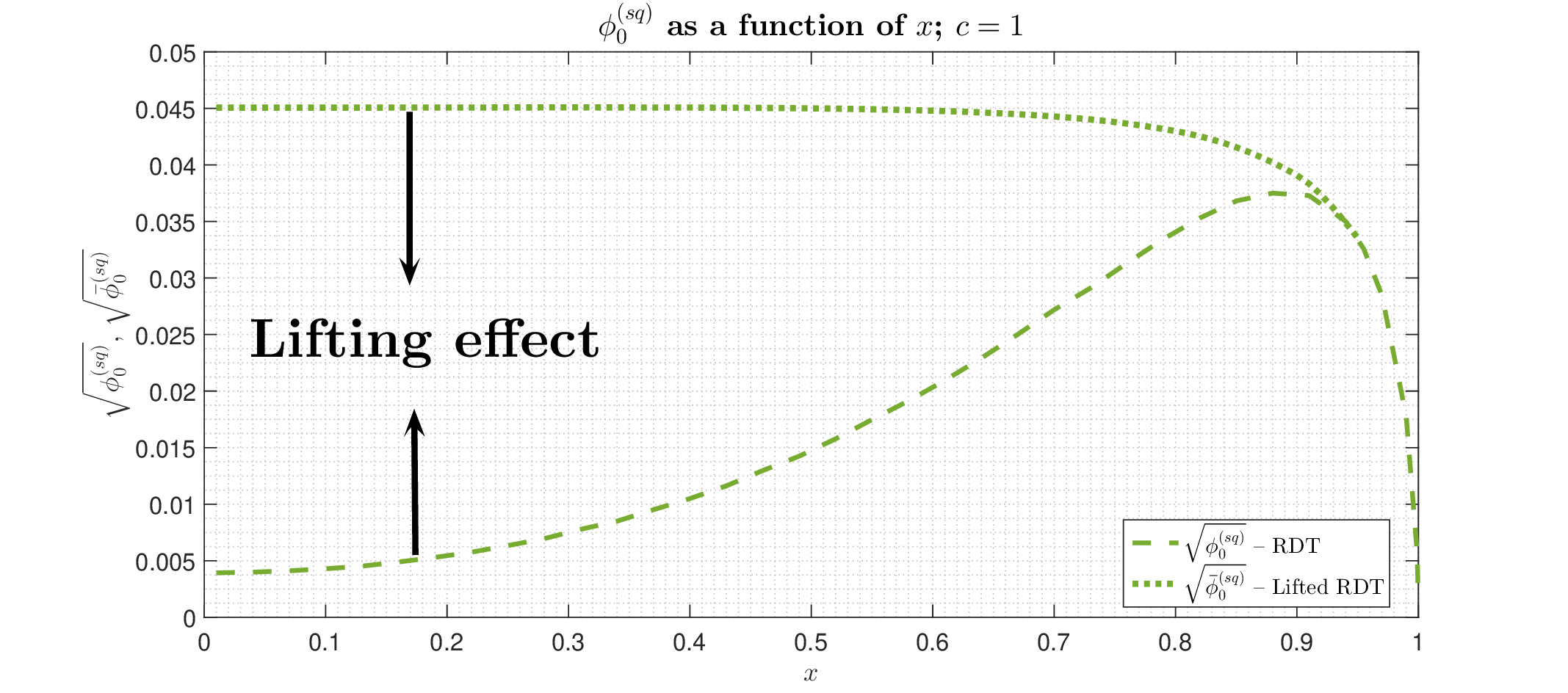}}
\caption{\emph{Squared magnitudes} -- effect of \emph{lifted} RDT ; $c=1$ and $\alpha=1.4$}
\label{fig:fig12}
\end{figure}

\section{Conclusion}
\label{sec:conc}

We considered the \emph{descending} phase retrieval algorithms and theoretically studied their performance. Relying on \emph{Random duality theory}  (RDT) we established a generic analytical program that allows precise statistical characterization of different algorithmic performance metrics. We introduced the concepts of \emph{parametric manifold} and its \emph{funneling points} and related them to the behavior of the underlying algorithms. In particular, we established an isomorphism between single funneling point manifolds structure  and generic global convergence of descending algorithms. Both plain and lifted RDT are then used to study the structure of the manifold. The impact of the sample complexity on the manifold structure is analyzed in detail and the emergence of a phase transition where manifold moves from a multi to a single funneling point structure is observed. The underlying isomorphism with the descending algorithms in return ensures that such a transition corresponds to having descending algorithms transition from the scenarios where they generically fail to the scenarios where they generically succeed in solving the phase retrieval.

We also developed and implemented a hybrid combination of a barrier gradient descent algorithm and a plain gradient descent. Despite the use of relatively small dimensions ($n=300$) (which is naturally  anticipated to imply strong manifold jitteriness effects), the  simulated and theoretical phase transitions are fairly close to each other. As the norm of the unknown vector  is one of the two manifold's parameters its critical role in both hybrid and plain gradient optimization processes is revealed and discussed in detail as well.

In a majority of our theoretical analyses we relied on the so-called non-squared magnitudes (basically amplitudes) objectives. They allowed for elegant closed form analytical considerations and made large portions of the underlying theoretical analyses more convenient. On the other hand, smoothness of the squared magnitudes derivatives was a bit more convenient for practical algorithmic implementations.  To ensure that practical implementations have their theoretical counterparts,  we conducted both RDT and lifted RDT theoretical analyses for squared magnitudes as well. Similar results are obtained as in non-squared case indicating that not much of a conceptual difference exists between the two..


A generic character of the developed methodologies allows for  many generalizations and further extensions. Clearly, from a theoretical point of view, studying further impact of fl RDT implementations is the first next logical step. On a more practical side though, the first next extensions include studying spectral initializers (since they are used as starting points in practical running), stability, robustness, and move to the complex domain. All concepts presented here can be used for any of these studies. The associated technical details are problem specific and we discuss them in separate companion papers.

\begin{singlespace}
\bibliographystyle{plain}
\bibliography{nflgscompyxRefs}
\end{singlespace}


\end{document}